\documentclass[twoside]{article}

\usepackage[accepted]{aistats2023}
\usepackage[round]{natbib}

\usepackage{gav_style}
\usepackage{microtype}
\usepackage{xr}

\begin{document}

%

%

\twocolumn[

\aistatstitle{Diffusion Generative Models in Infinite Dimensions}

\aistatsauthor{ Gavin Kerrigan \And Justin Ley \And  Padhraic Smyth }

\aistatsaddress{ University of California, Irvine \\ \texttt{gavin.k@uci.edu} \And University of California, Irvine \\ \texttt{jsley@uci.edu}\And University of California, Irvine \\ \texttt{smyth@ics.uci.edu}} ]

\begin{abstract}
  Diffusion generative models have recently been applied to domains where the available data can be seen as a discretization of an underlying function, such as audio signals or time series. However, these models operate directly on the discretized data, and there are no semantics in the modeling process that relate the observed data to the underlying functional forms. We generalize diffusion models to operate directly in function space by developing the foundational theory for such models in terms of Gaussian measures on Hilbert spaces. A significant benefit of our function space point of view is that it allows us to explicitly specify the space of functions we are working in, leading us to develop methods for diffusion generative modeling in Sobolev spaces. Our approach allows us to perform both unconditional and conditional generation of function-valued data. We demonstrate our methods on several synthetic and real-world benchmarks.
\end{abstract}

\section{Introduction}
\label{sect:introduction}

Diffusion models \citep{sohl2015deep, ho2020denoising} have recently emerged as a powerful class of generative models on a wide array of domains, ranging from images \citep{ho2020denoising, dhariwal2021diffusion,  saharia2022photorealistic, ramesh2022hierarchical} and video \citep{ho2022video, yang2022diffusion} to molecular conformation \citep{xu2022geodiff}. At an intuitive level, these methods work by iteratively perturbing the data distribution towards a tractable prior via additive Gaussian noise, and generation is performed by learning to undo this transformation.

Existing methods largely assume that the data distribution of interest is supported on a finite-dimensional Euclidean space. However, in many domains, the underlying signal is inherently \emph{infinite-dimensional}, where the observed data can be seen as a collection of discrete observations of some underlying function. Such datasets are often dubbed \emph{functional} \citep{ramsay2008functional}. For instance, a time series dataset consisting of the temperature collected at a particular location every $24$ hours can be seen as a uniform discretization of an underlying continuous-time temperature curve \citep{febrero2012statistical}.

Although diffusion models have empirically demonstrated strong performance on some functional domains, such as audio signals \citep{kong2021diffwave, chen2021wavegrad} and time series \citep{rasul2021autoregressive, tashiro2021csdi, alcaraz2022diffusion}, existing approaches work directly on an explicit discretization of the input space. It is thus unclear how existing methods relate to the underlying functions of interest. For instance, existing methods can not account for function-level assumptions about the data, such as continuity or smoothness constraints.

Motivated by a functional perspective, we propose a novel theoretical framework for diffusion generative modeling which operates directly in function space. Our primary contributions are as follows:
\begin{itemize}
    \item In Section \ref{sect:diffusion_models_fxn_space}, we develop a framework for diffusion generative modeling in terms of Gaussian measures on Hilbert spaces. Our method operates by adding Gaussian process noise directly to our infinite-dimensional functions. We learn to reverse this process by performing variational inference in function space, in which we minimize the KL divergence between a known Gaussian measure and a variational family of Gaussian measures.
    We discuss the necessary background on Gaussian measures in Section \ref{sect:background}.
    \item In Section \ref{sect:function_spaces}, we propose practical methods for approximating functional KL divergences by discretizing the underlying operators. The practical details depend heavily on the choice of function space, and we develop methods for the space of square-integrable functions as well as Sobolev spaces.
    \item In Section \ref{sect:experiments}, we empirically verify our framework on several synthetic and real-world benchmarks. In our experiments, our diffusion models are implemented via neural networks that parametrize mappings between function spaces, i.e. neural operators \citep{li2021fourier, li2020neural, kovachki2021neural}. We propose methods that allow for both unconditional and conditional generation of function-valued data. Importantly, our approach allows us to work with arbitrary non-uniform discretizations, thereby allowing us to train on datasets where the observation set varies across functions. Moreover, we are able to query our generated functions at arbitrary input locations. 
\end{itemize}

\section{Related Work}
\label{sect:related_work}

Diffusion models are most typically applied to data living in a Euclidean space having a fixed, finite dimension (e.g., see \cite{sohl2015deep, ho2020denoising, dhariwal2021diffusion, ho2020denoising} amongst others). More recent work has extended these methods to Riemannian manifolds, but still with a finite-dimensional assumption \citep{de2022riemannian, huang2022riemannian}. 

Most relevant to our work are diffusion models for signals, such as audio \citep{chen2021wavegrad, kong2021diffwave}, time series \citep{rasul2021autoregressive, tashiro2021csdi, alcaraz2022diffusion}, or neural processes \citep{dutordoir2022}. However, these current approaches for functional data all perform diffusion modeling by employing standard finite-dimensional diffusion modeling on the discretized functions. Concurrent to our work, \citet{bilos2022modeling} propose a diffusion model for temporal data, but do not take a function space perspective. As we will show in Section \ref{sect:existing_methods}, existing approaches can be viewed as special cases within the general theoretical framework we develop. 

Subsequent to our work in this paper, \citet{lim2023score} and \citet{pidstrigach2023infinite} in follow-up work proposed methodologies which are closely related and conceptually similar to our approach. As in our work, \citet{lim2023score} and \citet{pidstrigach2023infinite} both perturbed the function space distribution corresponding to the data via a trace-class Gaussian measure. Our work can be seen as extending the discrete time DDPM model \citep{ho2020denoising} to function spaces, while the works of \citet{lim2023score} and \citet{pidstrigach2023infinite} can be seen as extending score-matching techniques \citep{vincent2011connection, song2019generative} to function spaces. In particular, \citet{lim2023score} developed techniques for function space score matching in discrete time, and \citet{pidstrigach2023infinite} developed function space score matching techniques from a continuous time perspective.

Beyond diffusion models, a recent line of work has proposed deep generative models of functions \citep{garnelo2018neural, kim2018attentive, dupont2021generative, dupont2022functa}. In particular, generative models of functions based on neural operators have been proposed from a GAN approach \citep{rahman2022generative}. However, ours is the first work to combine diffusion models with neural operators. 
 
 Our approach is also broadly related to the general class of previous works that propose function-space perspectives in machine learning. In particular, such a point of view has proved useful for developing and understanding techniques used in Gaussian processes \citep{matthews2016sparse, wynne2022variational} and Bayesian deep learning \citep{sun2019functional, wild2022generalized, rudner2021rethinking, tran2022all, burt2021understanding}. Our work extends this function-space perspective to diffusion models.

\section{Notation and Background}
\label{sect:background}

We begin by setting up the notation for our problem and introducing the necessary background on Gaussian measures in Hilbert spaces, as well as their connection to the more familiar notion of Gaussian processes. In addition, we derive a closed-form expression for the KL divergence between Gaussian measures with equal covariance operators -- this KL divergence plays a key role in our approach. 

\subsection{Notation and Data}

Let $(\XX, \AA, \mu)$ be a measure space where $\XX \subseteq \R^{d_x}$, and let $\FF$ be a separable Hilbert space of measurable real-valued functions on $\XX$ with inner product $\langle \blank, \blank \rangle_{\FF}$. Note that we will often simply write $\langle \blank, \blank \rangle$ if the choice of $\FF$ is clear from context. We equip $\FF$ with its Borel $\sigma$-algebra $\BB(\FF)$. The prototypical example is $\XX = [0, 1]$ equipped with the Lebesgue measure and $\FF = L^2(\XX, \mu)$ equipped with its usual inner product $\langle f, g \rangle_{L^2(\XX,\mu)} = \int_\XX f g \d \mu$. However, our general framework is agnostic to the choice of $\FF$ -- see Section \ref{sect:function_spaces} for more details on this choice.

We assume that we have a dataset of the form $\DD = \{ u^{(1)}, u^{(2)}, \dots , u^{(n)} \}$, where each $u^{(j)} \in \FF$ is an i.i.d. draw from an unknown probability measure $\P_{\text{data}}$ on $\FF$. In practice, we typically only have noisy measurements of our functions on a finite subset of $\XX$. We let $\vec{x}^{(j)} = \{ x^{(1j)}, \dots, x^{(m j)} \} \subset \XX$ be a discrete subset of $\XX$ with corresponding observations $\vec{y}^{(j)} = \{ y^{(1j)}, \dots, y^{(mj)} \}$, where $y^{(ij)} = u^{(j)}(x^{(ij)}) + \epsilon^{(ij)}$ is the output of the unknown $j$th function $u^{(j)}$ at the $i$th observation point and $\epsilon^{(ij)}$ represents i.i.d. observation noise. Generally, both the location $\vec{x}^{(j)}$ and number $m = m_j$ of observation points may vary across the functions in our dataset.  

The focus of this work is to develop the theory and practice behind building a diffusion generative model for sampling from the function-space probability measure $\P_{\text{data}}$. 

\subsection{Gaussian Measures} 
\label{sect:gaussian_measures}

We now introduce some key background material on Gaussian measures \citep{da2014stochastic}. 

\begin{definition}
    Let $(\Omega, \BB, \P)$ be a probability space. A measurable function $F: \Omega \to \FF$ is called a \emph{Gaussian random element (GRE)} if for any $g \in \FF$, the random variable $\langle g, F \rangle$ has a Gaussian distribution on $\R$. The pushforward of $\P$ along $F$, denoted $\P_F = F_{\#} \P$ is a \emph{Gaussian (probability) measure} on $\FF$.
\end{definition}

Gaussian random elements $F \sim \P_F$ are random functions in $\FF$. Note that Gaussian measures exactly coincide with the standard notion of Gaussian distributions in the special case of $\FF = \R^n$ equipped with the usual inner product. 

For every GRE $F \sim \P_F$, there exists a unique mean element $m \in \FF$ given by

\begin{equation} \label{eqn:gaussian_measure_mean2}
    m = \int_{\FF} F \d \P_F.
\end{equation}

Similarly, there exists a unique linear covariance operator $C: \FF \to \FF$ given by
\begin{equation} \label{eqn:gaussian_measure_cov2}
    Cg = \int_\FF \langle g, F \rangle F \d \P_F - \langle g, m \rangle m  \quad \forall g \in \FF.
\end{equation}

A Gaussian measure is uniquely determined by its mean element and covariance operator. Note that for any $g \in \FF$, we have that $\langle g, F \rangle \sim \NN(\langle g, m \rangle, \langle Cg, g \rangle)$ follows a Gaussian distribution on $\R$ with mean $\langle g, m \rangle \in \R$ and variance $\sigma^2 = \langle Cg, g \rangle \in \R_{\geq 0}$ \citep{wild2022generalized}. 

The covariance operator $C$ is symmetric, positive semidefinite, and compact. Moreover, $C$ has finite trace, i.e. $\tr (C) = \E[||F||^2] < \infty$. Conversely, given any $m' \in \FF$ and any symmetric, positive semidefinite, trace-class linear operator $C': \FF \to \FF$, there exists a Gaussian measure having mean $m'$ and covariance operator $C'$. Thus, Gaussian measures are in one-to-one correspondence with their mean functions and covariance operators. We will write $\P_F = \NN(m, C)$ for such a Gaussian measure. We refer to \citet[Chapter~2]{da2014stochastic} and \cite{bogachev1998gaussian} for the proofs of these claims.

\subsection{KL Divergence between Gaussian Measures}

In our framework, we will perform variational inference in function space. However, one major challenge is that there is no analogue of the Lebesgue measure on infinite dimensional spaces \citep{eldredge2016analysis}, and so we must resort to a measure-theoretic definition of the KL divergence. To that end, for probability measures $\P, \Q$ on $\FF$, we define 
\begin{equation}
    \KL{\P}{\Q} = \int_{\FF} \log \left( \frac{\d \P}{\d \Q} \right) \d \P
\end{equation}
if $\P \ll \Q$, where $\d\P / \d \Q$ is the Radon-Nikodym derivative. We define this quantity to be infinite if $\P$ is not absolutely continuous with respect to $\Q$. 

We now consider the special case that $\P, \Q$ are Gaussian measures on $\FF$ with equal covariance operators. In this case, a version of the Feldman-H\'ajek Theorem gives us explicit control over the Radon-Nikodym derivative in terms of the parameters of $\P$ and $\Q$ \citep[Theorem~2.23]{da2014stochastic}. 

\begin{theorem}[The Feldman-H\'ajek Theorem]
\label{theorem:feldman_hajek_simple}
   Let $\P = \NN(m_1, C)$ and $\Q = \NN(m_2, C)$ be Gaussian measures on $\FF$ with equal covariance operators, and define $\Delta m = m_1 - m_2 \in \FF$. Then, $\P$ and $\Q$ are equivalent (i.e. mutually absolutely continuous) if and only if $\Delta m \in C^{1/2}(\FF)$. In this case, for any $f \in \FF$, the Radon-Nikodym derivative $\d \P / \d \Q$ is given by
   \begin{align}
       \frac{\d \P}{\d \Q}(f) = \exp &\big[ \left\langle \Delta m, C^{-1} (f - m_2) \right\rangle_{\FF} \nonumber \\
       &- \frac{1}{2} || C^{-1/2} \Delta m ||_{\FF}^2 \big],
   \end{align}
   where $C^{-1}$ is the pseudoinverse of $C$ and $C^{-1/2}$ is the pseudoinverse of $C^{1/2}$. 
\end{theorem}

 As a straightforward consequence of the Feldman-H\'ajek theorem, we derive a closed-form expression for the KL divergence between Gaussian measures with equal covariance operators. 
 
\begin{prop}
\label{prop:kl_gaussians}
Let $\P, \Q, \Delta m$ be defined as in Theorem \ref{theorem:feldman_hajek_simple}. Then,
    \begin{align}
    \KL{\P}{\Q} &= \frac{1}{2} \left\langle \Delta m, C^{-1} \Delta m \right\rangle_{\FF}.
    \end{align}
\end{prop}
\begin{proof}
    Appendix \ref{appendix_sect:gaussian_measures}.
\end{proof}

In Section \ref{sect:diffusion_models_fxn_space}, we make use of this result in order to develop diffusion models in function space. In Section \ref{sect:function_spaces}, we explore various practical methods for computing this functional KL divergence under various choices of the space $\FF$.

\subsection{Gaussian Processes}
\label{sect:gaussian_processes}

Gaussian processes (GPs) \citep{williams2006gaussian} are a popular class of models for specifying and learning distributions over functions. Formally, given a probability space $(\Omega, \BB, \P)$, a GP on $\XX$ is a jointly measurable map $G: \Omega \times \XX \to \R$ whose finite dimensional marginal distributions are Gaussian. 

In practice, a Gaussian process is typically specified by a mean function $m: \XX \to \R$ specifying $m(x) = \E[G(x)]$ and a kernel function $k: \XX^2 \to \R$ specifying the covariance structure of $G$ via $k(x, x') = \E[(G(x) - m(x)) (G(x') - m(x'))]$. We will write $G \sim \GP{m}{k}$ for such a Gaussian process.

Gaussian processes give us a practical way of specifying Gaussian measures, as we only need to specify a mean function and a kernel. The kernel $k$ plays an essential role in determining the sample path properties of a GP, such as continuity, differentiability, and periodicity \citep[Chapter~4]{williams2006gaussian}. In the case that $m \in \FF$ and $k$ is chosen such that $G \in \FF$ with probability one, we may identify $G$ with a GRE on $\FF$. For instance, if $m \in L^2(\XX, \mu)$ and $\int_{\XX} k(x, x) \d \mu (x) < \infty$, then we may identify $\GP{m}{k}$ with a Gaussian measure on $L^2(\XX, \mu)$. See \citet{wild2022generalized} and Section \ref{sect:function_spaces} for further details.

\section{Diffusion Models in Function Space}
\label{sect:diffusion_models_fxn_space}

Equipped with the necessary background, we now construct our diffusion generative model on $\FF$. Our construction mirrors that of DDPMs \citep{ho2020denoising}, with the key difference being that our diffusion process takes place in a space of infinite dimensions. We note that the constructions of \citet{ho2020denoising} rely heavily on properties of Gaussian densities in $\R^n$, and thus are not directly applicable to infinite-dimensional spaces as these spaces lack a reference measure from which to define such densities \citep{eldredge2016analysis}. Note further that $\FF = \R^n$ equipped with its usual inner product is a special case of our framework. 

\subsection{Forward Process}

We begin by defining the \emph{forward process}, a discrete-time Markov chain in $\FF$ which iteratively perturbs our data distribution $\P_{\text{data}}$ towards a fixed Gaussian measure $\NN(m, C)$. In what follows, we will choose $m = 0$ for simplicity. The choice of covariance operator $C$ is a hyperparameter which can be tuned.

We fix a finite number of timesteps $T \in \Z_{> 0}$ and a variance schedule $\beta: \{1, 2, \dots, T\} \to \R_{> 0}$, where we write $\beta_t$ for  $\beta(t)$. For any $u_0 \in \FF$, we iteratively sample from the forward process via 

\begin{equation} \label{eqn:forward_process}
    u_{t} = \sqrt{1 - \beta_t} u_{t-1} + \sqrt{\beta_t} \xi_t \qquad t = 1, 2, \dots, T
\end{equation}

where $\xi_t \sim \NN(0, C)$ are i.i.d. Gaussian random elements on $\FF$. 

Given a fixed value of $u_{t-1}$, our forward process gives us conditional probability measures ${\P_{t \mid t - 1}(\blank \mid u_{t-1})}$. We will write $\P_t$ for the marginal distribution on $\FF$ obtained at time step $t$ from this process, i.e. 

\begin{equation}
    {\P_{t}(\blank) = \int_{\FF} \P_{t \mid t - 1} (\blank \mid u_{t-1}) \d \P_{t-1} (u_{t-1})}]
\end{equation}

where $\P_0 = \P_{\text{data}}$. The value of $T$ and the variance schedule $\beta$ are chosen such that the final distribution is approximately equal to our specified Gaussian measure, i.e. $\P_T \approx \NN(0, C)$.

In the following proposition, we derive expressions for several distributions related to our forward process.

\begin{prop} \label{prop:forward_process_distributions}
Let $\gamma_t = \prod_{i=1}^t (1 - \beta_i)$. For the forward process defined in Equation \eqref{eqn:forward_process} with $m=0$ and fixed values of $u_0, u_{t-1}$:
\begin{equation}
    \P_{t \mid t - 1}(\blank \mid u_{t-1}) = \NN(\sqrt{1 - \beta_t}u_{t-1}, \beta_t C)
\end{equation}
\begin{equation}
    \P_{t \mid 0}(\blank \mid u_0) = \NN(\sqrt{\gamma_t} u_0, (1 - \gamma_t) C).
\end{equation}
 \end{prop}
\begin{proof}
    Appendix \ref{appendix_sect:gaussian_measures}.
\end{proof}

\subsection{Reverse Process and Loss}

Our generative model is then obtained by reversing the forward process, where we iteratively perturb the Gaussian measure $\NN(0, C)$ towards the data distribution $\P_0$. 

More specifically, to generate samples from our data distribution, we would like sample $u_T \sim \NN(0, C)$ and iteratively sample $u_{t-1} \sim \P_{t-1 | t}(\blank \mid u_t )$ from the time-reversal of our forward process for $t = T-1, \dots, 1$. However, while the posterior probability measure $\P_{t-1 \mid t}(\blank \mid u_t)$ is well-defined\footnote{This is because we assume $\FF$ is separable, which implies that $\FF$ is a Polish space. See \citet[Chapter~1]{ghosal2017fundamentals}. }, it is intractable. 

Most notably, using Bayes' rule here would require that the family of measures $\P_{t \mid t - 1}(\blank \mid u_{t-1})$ be simultaneously dominated by some fixed reference measure on $\FF$ for every choice of $u_{t-1}$. As these measures are Gaussian, the Feldman-H\'ajek theorem tells us that this is not possible (see Appendix \ref{appendix_sect:gaussian_measures}). Even if such technical difficulties were overcome (e.g. as in the Euclidean setting), computing Bayes' rule here would require computing an intractable normalization constant.

We instead take a variational approach, and approximate the posterior measures with a variational family of measures on $\FF$ parametrized by $\theta \in \R^p$. In particular, we set $\Q_T^\theta = \NN(0, C)$ and we approximate $ \P_{t-1 \mid t}(\blank \mid u_{t})$ by the Gaussian measure

\begin{equation}
    \Q^\theta_{t-1 \mid t}(\blank \mid u_t) = \NN\left(m_t^\theta(u_t), C_t^\theta(u_t)\right).
\end{equation}

Here, $m_t^\theta(u_t) = m_t^\theta(\blank \mid u_t) \in \FF$ is shorthand for a mean function in $\FF$ and $C_t^\theta(u_t) = C_t^\theta(\blank \mid u_t): \FF \to \FF$ is shorthand for a covariance operator. That is, the mean function and covariance operators depend on parameters $\theta$ as well as the timestep $t$ and function $u_t \in \FF$.

Although the reverse-time measures are intractable, the following proposition states that the reverse-time measures are tractable when conditioned on a starting function $u_0 \in \FF$. 
\begin{prop} \label{prop:conditional_reverse_measure}

Let $\gamma_t$ be defined as in Proposition \eqref{prop:forward_process_distributions}, and consider fixed values of $u_0, u_t \in \FF$. For $t = 2, 3, \dots, T$, let
$\widetilde{\beta}_t = \frac{1 - \gamma_{t-1}}{1 - \gamma_t}$ and let $\widetilde{m}_t(u_t, u_0) =  \widetilde{m}_t(\blank \mid u_t, u_0) \in \FF$ be defined by
\begin{align}
    \widetilde{m}_t(u_t, u_0) = &\frac{\sqrt{\gamma_{t-1}} \beta_t}{1 - \gamma_t} u_0 + \frac{\sqrt{1 - \beta_t} (1 - \gamma_{t-1})}{1 - \gamma_t} u_t. 
\end{align}

Then, $\P_{t-1 \mid t, 0}(\blank \mid u_t, u_0) = \NN(\widetilde{m}_t(u_t, u_0), \widetilde{\beta}_t C)$. 
\end{prop}
\begin{proof}
    Appendix \ref{appendix_sect:gaussian_measures}.
\end{proof}

We now tie our function-space Markov chain back to our observed data in order to obtain a loss function. Recall that our observations $\vec{y} \subset \R$ are assumed to be a vector of noisy observations of a function $u_0 \in \FF$ at some finite collection of points $\vec{x} \subset \XX$. We thus set the likelihood of our observed data to be $q^\theta(\mathbf{y} \mid \mathbf{x}, u_0) = \NN(\mathbf{y} ; u_0(\mathbf{x}), \sigma^2 I)$ where $\sigma^2 \in \R_{\geq 0}$ is some fixed constant. Note that $q^\theta$ is a Gaussian density on a finite dimensional space.

In the following proposition, we obtain a variational lower bound on the log-likelihood of our observations. This will serve as our loss function, which we seek to maximize over $\theta$. Although this lower bound is analogous to the standard DDPM lower bound \citep{ho2020denoising}, the proof is non-trivial as we must work directly with the underlying probability measures rather than their densities.

\begin{prop} \label{prop:loss_function}
The marginal likelihood of $\vec{y}$ given $\vec{x}$ is lower bounded by
    \begin{align}
        \log \; & q^\theta(\vec{y} \mid \vec{x}) \geq \\
        & \E_{\P} \bigg[\log q(\vec{y} \mid \vec{x}, u_0) - \KL{\P_T(\blank \mid \vec{x}, \vec{y})}{\Q_T^\theta(\blank)} \nonumber \\
        &- \sum_{t=1}^T \KL{\P_{t-1\mid t}(\blank \mid u_t, \vec{x}, \vec{y})}{\Q^\theta_{t-1\mid t}(\blank \mid u_t))} \bigg]. \nonumber
    \end{align}
\end{prop}
\begin{proof}
    Appendix \ref{appendix_sect:loss_function}.
\end{proof}

Since we assume $\Q_T^\theta$ has no trainable parameters, we may ignore the term $\KL{\P_T(\blank \mid \vec{x}, \vec{y})}{\Q_T^\theta(\blank)}$ during training.

\paragraph{Mean and Covariance Parametrization} We now make several further choices for our variational family. First, we analyze the terms
\begin{equation}
    L_{t-1} = \KL{\P_{t-1 \mid t}(\blank \mid u_t, u_0)}{\Q^\theta_{t-1\mid t}(\blank \mid u_t))}.
\end{equation}

Note that the first measure here is Gaussian by Proposition \eqref{prop:conditional_reverse_measure}, and the second is Gaussian by assumption. A more general form of the Feldman-H\'ajek theorem (see Appendix \ref{appendix_sect:gaussian_measures}) places strict requirements on the corresponding covariance operators in order to obtain a finite KL divergence. In particular, the term $L_{t-1}$ will be infinite if
\begin{equation}
    \widetilde{\beta}_t^{-1} \left( C^{-1/2} C_t^\theta(u_t)^{1/2} \right) \left( C^{-1/2} C_t^\theta(u_t)^{1/2} \right)^* - I
\end{equation}
is not a Hilbert-Schmidt operator on the closure of $C^{1/2}(\FF)$. For instance, even the seemingly innocuous choice of $C_t^\theta(u_t) = \alpha \widetilde{\beta}_t C$ for any non-negative $\alpha \neq 1$
will result in an infinite KL divergence. Thus, motivated by necessity, we will choose $C_t^\theta(u_t) = \widetilde{\beta}_t C$.

Under this choice of $C_t^\theta(u_t)$, a consequence of Propositions \eqref{prop:kl_gaussians} and \eqref{prop:conditional_reverse_measure} is that
\begin{align} \label{eqn:loss_kl_term_no_simplification}
   L_{t-1}= \frac{1}{2 \widetilde{\beta}_t}|| C^{-1/2} (\widetilde{m}_t(u_t, u_0) - m_t^\theta(u_t)) ||_{\FF}^2.
\end{align}

Similar to DDPM \citep{ho2020denoising}, we further choose to parametrize the mean function via
\begin{equation}
    \label{eqn:mean_parametrization}
    m_t^\theta(u_t) = \frac{1}{\sqrt{1 - \beta_t}}\left( u_t - \frac{\beta_t}{\sqrt{1 - \gamma_t}} \xi_t^\theta(u_t) \right)
\end{equation}

where $\xi_t^\theta(u_t) \in \FF$ is the output of a model parametrized by $\theta$ which takes in $(t, u_t)$ as inputs and has function-valued outputs. In other words, our model is a parametrized mapping $   \xi^\theta: \{1, 2, \dots, T\} \times \FF \to \FF$ specified via $(t, u_t) \mapsto \xi_t^\theta(\blank \mid u_t)$. Under this choice, we have that
\begin{align}
    \label{eqn:loss_kl_term}
    L_{t-1} = \lambda_t || C^{-1/2} (\xi_t - \xi_t^\theta(u_t) ) ||_{\FF}^2
\end{align}

where $\lambda_t = \beta_t^2 / (2 \widetilde{\beta}_t (1 - \beta_t) (1 - \gamma_t)) \in \R$ is a time-dependent constant. See Appendix \ref{appendix_sect:loss_function} for details. In light of Proposition \eqref{prop:kl_gaussians}, we see that $L_{t-1}$ is (up to a multiplicative constant) the KL divergence between two Gaussian measures on $\FF$ having covariance operators $C$ and respective means $\xi_t, \xi_t^\theta(u_t)$.  As is standard in diffusion generative modeling, we drop the constant $\lambda_t$ when training in order to obtain a re-weighted variational lower bound \citep{ho2020denoising} for improved quality.

In Section \ref{sect:experiments}, we provide a practical instantiation of the mapping $\xi_t^\theta$ via neural operators \citep{li2021fourier, li2020neural, kovachki2021neural}.

Following our work, \citet{lim2023score} noted that the parametrization of the loss given in Equations \ref{eqn:mean_parametrization} and \ref{eqn:loss_kl_term} results in an infinite quantity when the dimension of $\FF$ is infinite. However, it is straightforward to remedy this by considering an alternative parametrization, where the model directly predicts a rescaled version of $u_0$ rather than predicting $\xi_t$, e.g., see Appendix E and Appendix I of \citet{lim2023score} for additional details. In our experiments in this paper we used the parametrization given in Equations \ref{eqn:mean_parametrization} and \ref{eqn:loss_kl_term}, and note that the corresponding quantities are only infinite in the limit corresponding to a discretization size of zero.

\section{Function Spaces and KL Approximations}
\label{sect:function_spaces}

We have thus far described our framework in terms of abstract Gaussian measures on Hilbert spaces. We can obtain a concrete instantiation of our framework by choosing an appropriate space of functions to work on, as well as a choice of Gaussian measure which specifies our forward process. 

In this section, we explore two choices for $\FF$: the space of square-integrable functions $L^2(\XX,  \mu)$ and the Sobolev spaces $H^k(\XX, \mu) = W^{k,2}(\XX, \mu)$. We derive practical methods for estimating the KL divergence between Gaussian measures in these spaces, which is necessary for evaluating the terms in our loss function given in Equation \eqref{eqn:loss_kl_term}.

To compute the functional KL divergence in Proposition \eqref{prop:kl_gaussians}, we derive discrete approximations of both the inverse covariance operator $C^{-1}$ and the associated inner product. Suppose that $m_1$ and $m_2$ are known on a common discretization $\vec{x} = \{x^{(1)}, \dots, x^{(n)} \} \subset \XX$ which is drawn from the measure $\mu$ on $\XX$.  For any function $f: \XX \to \R$, we write $f(\vec{x}) \in \R^n$ to represent the vector corresponding to evaluating $f$ at the points contained in $\vec{x}$. We assume further that our Gaussian measure $\xi \sim \GP{0}{k}$ is specified by a mean-zero Gaussian process with kernel $k$, with appropriate restrictions on $k$ such that $\xi \in \FF$ (see Section \ref{sect:gaussian_processes}).  In Appendix \ref{appendix_sect:addl_experiments}, we explore estimating these KL divergences with spectral methods, but find that it is sensitive to the discretization size, even when the eigenfunctions are analytically known.

\begin{figure*}[!!t]
    \centering
    \includegraphics{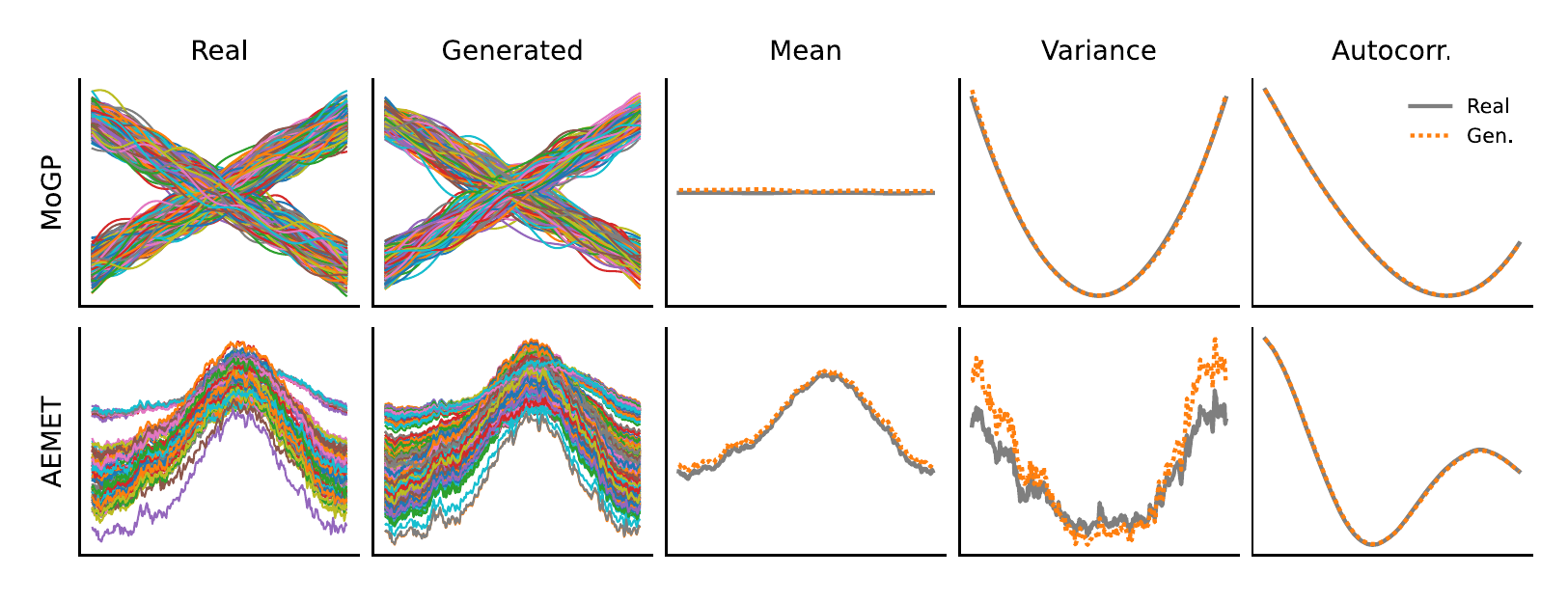}
    \caption{Unconditional function generation on a synthetic (MoGP) and real-world (AEMET) dataset. For each dataset, a GNO model was trained on the plotted functions (first column), and a total of $500$ functions were sampled from the model (second column). The generated curves closely match the training curves in both perceptual quality and pointwise statistics.}
    \label{fig:unconditional}
\end{figure*}
\subsection{Square-Integrable Functions}

We first consider the space $\FF = L^2(\XX, \mu)$ of measurable, square-integrable functions $f: \XX \to \R$ equipped with the inner product $\langle f, g \rangle_{L^2} = \int_\XX f g \d \mu$.


For many applications, this is a natural choice of function space as square integrability is a relatively weak assumption. Moreover, $p=2$ is the unique choice such that the Banach space $L^p(\XX, \mu)$ is also a Hilbert space, and the associated inner product structure is a useful tool for performing calculations.

In $L^2(\XX, \mu)$, the covariance operator associated with our kernel function $k$ can be explicitly described via
\begin{equation}
    [Cg](x) = \int_{\XX} k(x, x') g(x') \d \mu(x') \qquad \forall g \in \FF.
\end{equation}

We provide a derivation of this formula for $C$ in Appendix \ref{appendix_sect:covariance_operators}. Let $K_{\vec{x}\vec{x}} \in \R^{n \times n}$ be the covariance matrix specified by $k$ and evaluated on $\vec{x}$, i.e. the $(i,j)$th entry of $K_{\vec{x}\vec{x}}$ is given by $k(x_i, x_j)$. Then, upon replacing $\mu$ with the empirical measure specified by $\vec{x}$, we have ${[Cg]\left(\vec{x}\right) \approx n^{-1} K_{\vec{x}\vec{x}} g(\vec{x}) \in \R^n}$, so that the (scaled) covariance matrix $K$ is a discrete approximation of the covariance operator $C$ which may inverted. Replacing $\mu$ once more with the empirical measure specified by $\vec{x}$, we then have 

\begin{equation} \label{eqn:l2_kl_approx}
   \KL{\NN(m_1, C)}{\NN(m_2, C)} \approx \frac{1}{2} \Delta m (\vec{x})^T K^{-1}_{\vec{x}\vec{x}} \Delta m(\vec{x}).
\end{equation}

Interestingly, this is precisely the KL divergence between two finite-dimensional Gaussians with equal covariance matrices $K_{\vec{x}\vec{x}}$ and means $m_1(\vec{x}), m_2(\vec{x})$.

Furthermore, we note that \citet{sun2019functional} prove that the KL divergence between two stochastic processes is the supremum of the KL divergences between their finite-dimensional marginals. Our approximation in Equation \eqref{eqn:l2_kl_approx} is increasing under refinements of the observation set $\vec{x}$, and thus is a lower bound on the true KL divergence.

\begin{prop} \label{prop:kl_increasing}
   Equation \eqref{eqn:l2_kl_approx} is strictly increasing under refinements of the observation set $\vec{x}$. In particular, if $\vec{z} \subset \vec{x}$, then
   \begin{equation}
       \Delta m (\vec{z})^T K_{\vec{z}\vec{z}}^{-1} \Delta m (\vec{z}) \leq \Delta m (\vec{x})^T K_{\vec{x}\vec{x}}^{-1} \Delta m (\vec{x}). 
   \end{equation}
\end{prop}
\begin{proof}
    Appendix \ref{appendix_sect:covariance_operators}.
\end{proof}

\subsection{Sobolev Spaces}

A second choice of function spaces that have many practical applications are the Sobolev spaces $H^k(\XX, \mu)$ consisting of functions in $L^2(\XX,\mu)$ whose mixed $\alpha$th-order partial derivatives of order at most $k$ exist (in a weak sense) and are also in $L^2(\XX, \mu)$ \citep[Chapter~5]{evans2010partial}. Of particular interest is the setting where $\XX \subset \R$ and $k=1$, where the inner product is given by
\begin{equation}
    \langle f, g \rangle_{H^1} = \langle f, g \rangle_{L^2} + \langle \partial_x f, \partial_x g \rangle_{L^2}. 
\end{equation}

When the Gaussian process associated with  the kernel function $k$ lies in $H^1$ with probability one, the corresponding covariance operator can be expressed as
\begin{align}
    [Cg](x) &= \int_\XX k(x, x') \d \mu(x') \\
    &+ \int_{\XX} \left[ \partial_{x'} k(x, x') \right] \left[ \partial_{x'} g(x') \right] \d \mu(x'). \nonumber
\end{align}

See Appendix \ref{appendix_sect:covariance_operators} for a derivation. Our discretization in this setting follows closely that of our techniques for the space $L^2(\XX, \mu)$, with the additional necessity of employing a discrete differential operator. To that end, let $D \in \R^{n \times n}$ be any discrete approximation to the first-order differentiation operator. In practice
 we use a discretization based on finite-difference equations. Let $K_{\vec{x}\vec{x}}' \in \R^{n \times n}$ be the covariance matrix corresponding to the differeniated kernel $\partial_{x'} k(x, x')$. That is, the $(i, j)$th entry of $K'_{\vec{x}\vec{x}}$ is given by $\frac{\partial }{\partial x'} k(x_i, x_j)$. Then, the covariance operator $C$ can be discretized via ${[Cg](\vec{x}) \approx n^{-1} \left[K_{\vec{x}\vec{x}} + K'_{\vec{x}\vec{x}} D\right] g(\vec{x}) \in \R^n}$, and moreover,
\begin{align}
     \KL{&\NN(m_1, C)}{\NN(m_2, C)} \approx \\
     &\frac{1}{2}\Delta m(\vec{x})^T \left[ I + D^T D \right] \left[ K_{\vec{x}\vec{x}} + K'_{\vec{x}\vec{x}} D \right]^{-1} \Delta m (\vec{x}). \nonumber
\end{align}

Although the covariance operator $C$ is guaranteed to be positive semidefinite in theory, discretizing this operator often results in a non-PSD matrix approximation which may cause training to diverge. In practice, we project the matrix $[I + D^T D][K_{\vec{x}\vec{x}} + K'_{\vec{x}\vec{x}}D]^{-1}$ to the nearest symmetric PSD matrix (in terms of the Frobenius norm) \citep{higham1988computing, cheng1998modified}. See Appendix \ref{appendix_sect:covariance_operators} for details.

\subsection{Existing Methods in Terms of Our Theory}
\label{sect:existing_methods}

In terms of our methodology, existing methods \citep{kong2021diffwave, chen2021wavegrad, tashiro2021csdi, dutordoir2022} can be viewed as operating in the space $\FF = L^2(\XX, \mu)$, with the discretization employed in Equation \eqref{eqn:l2_kl_approx}. In all of these methods, the forward process is defined via a white noise prior. However, such a prior can \emph{not} be seen as a Gaussian measure. In particular, the white noise process is not jointly measurable \citep[Example~1.2.5]{kallianpur2013stochastic}, and thus one is unable to consider the corresponding sample paths as elements of some function space. A GRE corresponding to this prior would have infinite variance, as the corresponding covariance operator would be the identity operator. Nonetheless, despite these foundational concerns, existing methods show strong empirical performance. Explaining this performance from a functional point of view, for example through the theory of generalized functions \citep{grubb2008distributions}, is an interesting challenge for future work.

\section{Experiments} \label{sect:experiments}


In this section, we perform several experiments in order to illustrate how our theoretical framework can be implemented as a  practical estimation methodology. In all experiments, we parametrize $\xi_t^\theta(u_t)$ via a graph neural operator (GNO) \citep{li2020neural, kovachki2021neural}. See Appendix \ref{appendix_sect:model_details} for our model configurations and hyperparameter settings. Our models are trained by minimizing the reweighted negative ELBO as described in Section \ref{sect:diffusion_models_fxn_space}. In all plots, our functional diffusion model is denoted \emph{FuncDiff}. Pseudocode and additional details for all of our algorithms is available in Appendix \ref{appendix_sect:pseudocode}. 

Code for all of our experiments is available at {\url{https://github.com/GavinKerrigan/functional_diffusion}}.

A key property of the GNO is the ability to condition on arbitrary discretizations of $\XX$. This allows us to train our models on functions that are observed at different points, as well as to condition on arbitrary function observations when performing conditional generation. Moreover, as neural operators parametrize mappings between function spaces, we are able to query our model at arbitrary input locations. Thus, our model is not tied to any particular discretization.

\paragraph{Datasets} We use both a synthetic and a real-world dataset in the main paper to illustrate our approach, with results on additional real-world datasets in Appendix \ref{appendix_sect:addl_experiments}. Our synthetic dataset is a mixture of Gaussian processes (\emph{MoGP}) with a squared-exponential kernel with variance $\sigma^2 = 0.4$ and length scale $\ell = 0.1$, where the first mixture component has mean $m_1 = 10x - 5$ and the second has mean $m_2 = -10x + 5$. These functions are observed on a uniform discretization of $[0, 1] \subset \R$. We use $64$ observation points unless otherwise specified. Our real-world dataset (\emph{AEMET}) is a well-known dataset in the functional data analysis literature. This dataset consists of 73 curves, where each curve is the mean daily temperature at a particular Spanish weather station, so that each curve has a total of $365$ discrete observations \citep{febrero2012statistical}. See Figure \ref{fig:unconditional} for an illustration of these datasets.

\subsection{Unconditional Generation}

In this experiment, we sample curves unconditionally from our trained model. In Figure \ref{fig:unconditional}, the generated curves closely match the training data in terms of perceptual qualities. We additionally compute the pointwise mean, pointwise variance, and mean autocorrelation of both the real and generated curves. The summary statistics of the generated data closely match those of the real data, indicating that the model has successfully learned to sample from the functional distribution. See Appendix \ref{appendix_sect:addl_experiments} for a comparison to a simple baseline based on functional PCA \citep[Chapter~6]{ramsay2008functional} and additional datasets.

\subsection{Conditional Generation}

Our proposed approach for conditional generation is an extension of the ILVR method \citep{choi2021ilvr} to functional data. This method works by perturbing conditioning information via the forward process, and during generation we set the values of the generated function at the conditioning locations to these perturbed values. In particular note that we are able to condition a pre-trained unconditional model on arbitrary function observations. Thus, this method may potentially be applied to a wide array of tasks, such as extrapolation, upsampling, or data imputation. 

In Figure \ref{fig:conditional}, we demonstrate this by conditioning our generation on a known segment of the function. We see that our method is able to leverage the learned functional distribution in order to accurately extrapolate the given conditioning information. We compare to a Gaussian process regression (\emph{GPR}) baseline, where we fit a Gaussian process only to the conditioning information. 
Unsurprisngly, the GPR method is not able to accurately extrapolate the conditioning information, as it has no additional information regarding the underlying functional distribution.


\begin{figure}[!ht]
    \centering
    \includegraphics{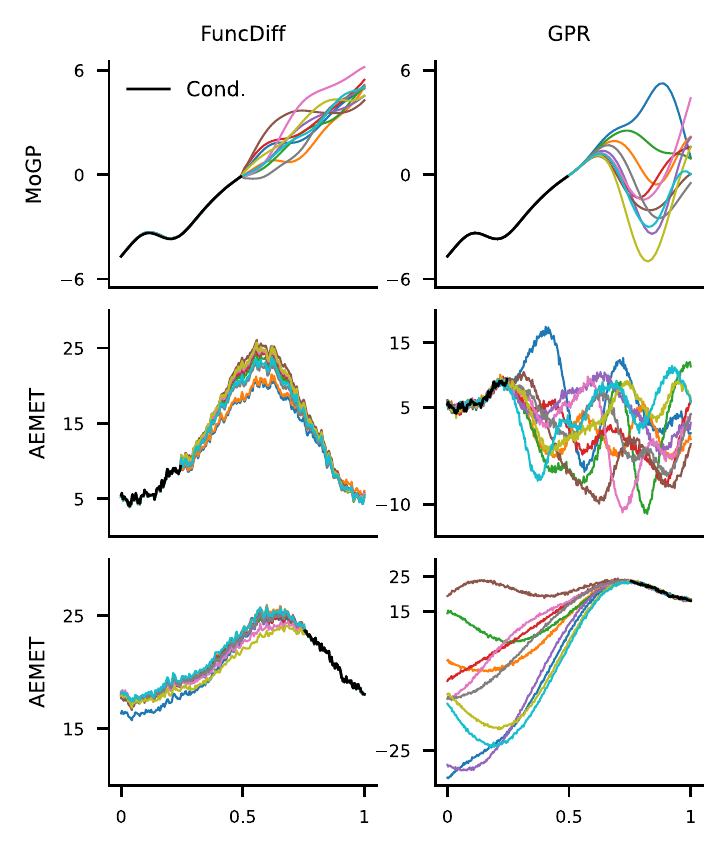}
    \caption{ Conditional samples of our model (FuncDiff) are compared against Gaussian process regression (GPR). In each plot, both models are conditioned on the black curves.}
    \label{fig:conditional}
\end{figure}

\begin{figure}[!htb]
    \centering
    \includegraphics{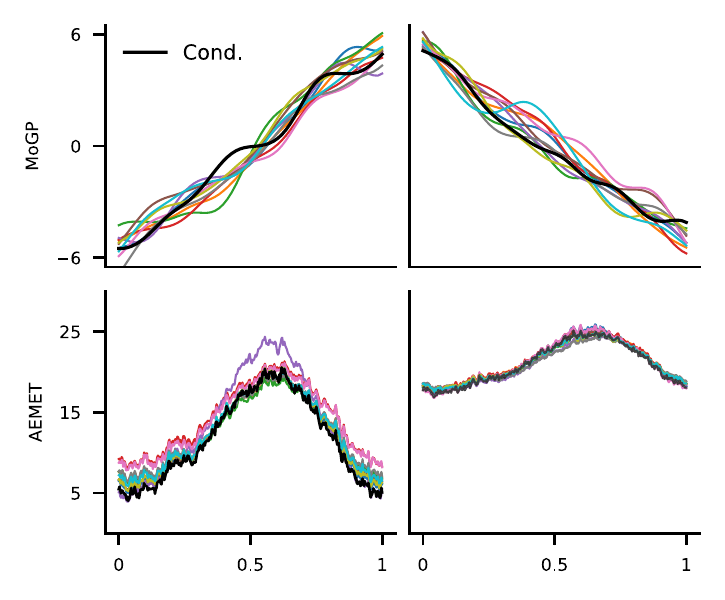}
    \caption{An illustration of our soft conditioning method. We condition the generative process on the black curves for all but the final $150$ diffusion steps. This allows us to generate functions that are qualitatively similar to the given conditioning information (in black), such that the generated function values do not necessarily exactly match those of the conditioning information.}
    \label{fig:soft_conditional}
\end{figure}

Moreover, our conditioning method allows us to do \emph{soft conditioning}, where the diffusion process is not conditioned on the observed values for some number of the final diffusions steps. This allows us to generate curves that are similar to a given observation, but not exactly matching. For example, this can be used to select a particular mode to sample from in a multimodal dataset. We demonstrate this in Figure \ref{fig:soft_conditional}. As a potential future application, soft conditioning could be applied as a data augmentation method for functional data. 


\subsection{Function Spaces}

Lastly, we experiment with the choice of function space. In particular, we compare the use of the $L^2(\XX, \mu)$ inner product against the use of the $H^1(\XX, \mu)$ inner product. Intuitively, the derivative term in Sobolev inner product will penalize generated functions that are not smooth. In Table 2, we measure the smoothness of generated curves by computing the mean standard deviation of the derivatives of said curves. We find empirically that using the Sobolev loss can result in significantly smoother generations when the underlying functional dataset is highly regular. As smoothness is not a desirable property for the AEMET dataset, we include here a dataset consisting of linear functions (\emph{Linear}) instead. See Appendix \ref{appendix_sect:addl_experiments} for more on this dataset. We use the Mat\'ern kernel with $\nu = 3/2$ when working with the the Sobolev norm as this kernel has differentiable sample paths.

\begin{table}[!ht]
    \centering
    \caption{Mean smoothness of generated functions as measured by the standard deviation of the function derivatives, averaged across $500$ samples. Using the Sobolev norm over the $L^2$ norm can significantly increase the smoothness of generated functions, while not harming performance if the generated functions are already sufficiently smooth. \vspace{3.5mm}}
    
    \begin{tabular}{lcc}
        \toprule
        Dataset & $L^2(\XX, \mu )$ & $H^1(\XX, \mu)$ \\
        \midrule 
        Linear & 0.753 & 0.203 \\
        MoGP  &  24.73 & 24.74 \\
    \end{tabular}
    \label{tab:l2_vs_sobolev}
\end{table}

\section{Conclusion}

We propose a framework for diffusion generative modeling in infinite-dimensional spaces of functions and develop practical techniques for realizing this framework on real-world data. Enabled by our framework, future functional diffusion models may be able move beyond the typical $L^2$-space assumption in order to incorporate informative prior information. A remaining challenge for functional diffusion models is to consider the continuous-time limit and elucidating connections with score-based methods \citep{song2021score}. For instance, it may be possible to view continuous-time functional diffusion models as stochastic PDEs, potentially enabling more efficient sampling methods. 

\subsubsection*{Acknowledgements}

This research was supported in part by the National Science Foundation under award number  1900644, by the HPI Research Center in Machine Learning and Data Science at UC Irvine,  and by a Qualcomm Faculty award.

\bibliography{refs_clean}

\setcounter{section}{0}
\renewcommand{\thesection}{A.\arabic{section}}
\onecolumn
\aistatstitle{Diffusion Generative Models in Infinite Dimensions: \\
Supplementary Materials}

\section*{Summary of the Appendix}

The appendix is organized as follows:
\begin{itemize}
    \item In Section \ref{appendix_sect:gaussian_measures}, we discuss additional properties of Gaussian measures not contained in the main paper. In particular, we show that the space of Gaussian measures is closed under affine transformations and independent sums. We leverage these facts to prove Propositions \eqref{prop:forward_process_distributions}, \eqref{prop:conditional_reverse_measure}. In addition, we provide a statement of the general formulation of the Feldman-H\'ajek theorem, and use this theorem to prove Proposition \eqref{prop:kl_gaussians}.
    \item In Section \ref{appendix_sect:loss_function}, we provide a proof of the ELBO in Proposition \eqref{prop:loss_function}, as well as a detailed derivation of the reparametrization and simplified loss of Section \ref{sect:diffusion_models_fxn_space}.
    \item In Section \ref{appendix_sect:covariance_operators}, we derive expressions for the covariance operator associated with a Gaussian process under the assumptions that $\FF = L^2(\XX, \mu)$ or $\FF = H^1(\XX, \mu)$. In addition, we prove Proposition \eqref{prop:kl_increasing}.
    \item In Section \ref{appendix_sect:model_details}, we provide the details of the model architectures and training procedures used in our experiments, as well as an ablation on the choice of kernel in the forward process.
    \item In Section \ref{appendix_sect:pseudocode} we provide pseudocode for training our model as well as both unconditional and conditional sampling.
    \item In Section \ref{appendix_sect:addl_experiments}, we include additional experiments not detailed in the main paper. These include unconditional generation results on additional datasets, a comparison to a simple baseline based on functional PCA, and a comparison between the discrete approximations of the KL divergence proposed in Section \ref{sect:function_spaces} to a method based on computing the spectrum of the corresponding covariance operator.
\end{itemize}

\vfill
\newpage

\section{Gaussian Measures}
\label{appendix_sect:gaussian_measures}

In this section, we include some useful facts and additional details regarding Gaussian measures, as well as proofs of Propositions \eqref{prop:kl_gaussians}, \eqref{prop:forward_process_distributions}, \eqref{prop:conditional_reverse_measure}.  

\subsection{Basic Properties}

\begin{lemma}[Affine Transformations of GREs] \label{lemma:linear_gre}
    Let $F \sim \NN(m, C)$ be a GRE on $\FF$. Then, for $\alpha \in \R$ and $g \in \FF$, we have that $\alpha F + g \sim \NN(\alpha m + g, \alpha^2 C)$. 
\end{lemma}
\begin{proof}
    Fix any $h \in \FF$, and note that
    \begin{equation}
        \langle h, \alpha F + g \rangle = \alpha \langle h, F \rangle + \langle h, g \rangle.
    \end{equation}
    
    Since $\langle h, F \rangle \sim \NN( \langle h, m \rangle, \langle Ch, h \rangle)$, it follows that $\langle h, \alpha F + g \rangle$ must follow a Gaussian distribution on $\R$ with mean
    \begin{equation}
        \alpha \langle h, m \rangle + \langle h, g \rangle = \langle h, \alpha m + g \rangle
    \end{equation}
    and variance
    \begin{equation}
        \alpha^2 \langle Ch, h \rangle = \langle \alpha^2 C h , h \rangle.
    \end{equation}
    
    Thus, we have shown that $\alpha F + g$ is a GRE on $\FF$, as its inner product with arbitrary $h \in \FF$ is Gaussian on $\R$. Moreover, we have computed its mean and covariance operator as claimed.
\end{proof}

\begin{lemma}[Sum of Independent GREs]
\label{lemma:sum_of_gres}

If $F \sim \NN(m_1, C_1)$ and $G \sim \NN(m_2, C_2)$ are independent GREs on $\FF$, then $F + G \sim \NN(m_1 + m_2, C_1 + C_2)$.
\end{lemma}
\begin{proof}
    Let $Z = F + G$. Write $\P_F, \P_G, \P_Z$ for the probability measures of $F, G, Z$ respectively. The Fourier transform of $\P_F$ is given by
    \begin{equation}
        \widehat{\P}_F (\lambda) = \int_{\FF} \exp \left[i \langle \lambda, F \rangle \right] \d \P_F \qquad \forall \lambda \in \FF,
    \end{equation}
    and is given analogously for our other measures. By \citet[A.3.17]{bogachev1998gaussian} and the subsequent discussion, a probability measure is uniquely determined by its Fourier transform. Moreover, we have that
    \begin{equation}
        \widehat{\P}_Z (\lambda) = \widehat{\P}_F (\lambda)  \widehat{\P}_G(\lambda) \qquad \forall \lambda \in \FF.
    \end{equation}
    
    Using the expression for the Fourier transform of a Gaussian measure given in \citet[Chapter~2]{da2014stochastic}, we see that that
    \begin{align}
        \hat{P}_Z(\lambda) &= \exp\left[i \langle \lambda, m_1 \rangle - \frac{1}{2} \langle C_1 \lambda, \lambda  \rangle \right] \exp\left[i \langle \lambda, m_2 \rangle - \frac{1}{2} \langle C_2 \lambda, \lambda  \rangle \right] \\
        &= \exp \left[i \langle \lambda, m_1 + m_2 \rangle - \frac{1}{2} \langle (C_1 + C_2) \lambda, \lambda \rangle \right]
    \end{align}
    
    which is precisely the Fourier transform of the measure $\NN(m_1 + m_2, C_1 + C_2)$. 

\end{proof}

\subsection{Diffusion Process Measures}

In this subsection we derive various closed-form measures related to our diffusion process in Section \eqref{sect:diffusion_models_fxn_space}. 

\textbf{Proof of Proposition \eqref{prop:forward_process_distributions}.}
\begin{proof}
The first and second claims are special cases of Lemma \eqref{lemma:linear_gre}.

For the third claim, we proceed by induction on $t$. The case $t = 1$ is clear from Lemma \eqref{lemma:linear_gre}. Now, suppose
\begin{equation}
    u_{t-1} \mid u_0 \sim \NN(\sqrt{\gamma_{t-1}}u_0, (1 - \gamma_{t-1})C).
\end{equation}

By the definition of the forward process and our inductive assumption, we have that $u_t = \sqrt{1 - \beta_t}u_{t-1} + \sqrt{\beta_t} \xi_t$ is the sum of two independent GREs: the first is
\begin{equation}
   \sqrt{1 - \beta_t}u_{t-1} \sim \NN(\sqrt{\gamma_t} u_0, (1 - \beta_t)(1 - \gamma_{t-1})C)
\end{equation}
and the second is $\sqrt{\beta_t} \xi_t \sim \NN(0, \beta_t C)$. By Lemma \eqref{lemma:sum_of_gres}, we obtain the result, as
\begin{equation}
    (1 - \beta_t)(1 - \gamma_{t-1}) + \beta_t = 1 - (1 - \beta_t)\gamma_{t-1} = 1 - \gamma_t. 
\end{equation}
\end{proof}

\textbf{Proof of Proposition \eqref{prop:conditional_reverse_measure}.} 
\begin{proof}
By Proposition \eqref{prop:forward_process_distributions} and Lemma \eqref{lemma:linear_gre}, we may write
\begin{equation} \label{eqn:step1}
    u_{t-1} = \sqrt{\gamma_{t-1}}u_0 + \sqrt{1 - \gamma_{t-1}} \xi \qquad \xi \sim \NN(0, C)
\end{equation}
and by construction we have
\begin{equation} \label{eqn:step2}
    u_t = \sqrt{1 - \beta_t} u_{t-1} + \sqrt{\beta_t} \xi' \qquad \xi' \sim \NN(0, C)
\end{equation}

where $\xi, \xi' \sim \NN(0, C)$ are independent GREs. Our strategy is to manipulate these expressions to obtain a reparametrized expression for $u_{t-1}$. By Equation \eqref{eqn:step1},
\begin{equation} \label{eqn:step3}
    \beta_t \sqrt{\gamma_{t-1}} u_0 = \beta_t \left[u_{t-1} - \sqrt{1 - \gamma_{t-1}} \xi \right],
\end{equation}
and similarly by Equation \eqref{eqn:step2},
\begin{equation} \label{eqn:step4}
    (1 - \gamma_{t-1}) \sqrt{1 - \beta_t} u_t = (1 - \gamma_{t-1}) \left[ (1 - \beta_t) u_{t-1} + \sqrt{\beta_t} \sqrt{1 - \beta_t} \xi' \right].
\end{equation}

Upon summing Equations \eqref{eqn:step3}-\eqref{eqn:step4} and isolating the $u_{t-1}$ terms,
\begin{align}
    \left(\beta_t + (1 - \gamma_{t-1})(1 - \beta_t) \right) u_{t-1} &= \beta_t \sqrt{\gamma_{t-1}} u_0 + (1 - \gamma_{t-1}) \sqrt{1 - \beta_t} u_t \\
    &+ \beta_t \sqrt{1 - \gamma_{t-1}} \xi - (1 - \gamma_{t-1}) \sqrt{\beta_t}\sqrt{1 - \beta_t} \xi' \nonumber.
\end{align}

On the LHS, we have
\begin{align} \label{eqn:step5}
    (\beta_t + (1 - \gamma_{t-1})(1 - \beta_t)) u_{t-1} &= (\beta_t + (1 - \beta_t) - (1-\beta_t)(\gamma_{t-1}))u_{t-1} \nonumber \\
    &= (1 - \gamma_t) u_{t-1}
\end{align}

thereby allowing us to obtain
\begin{align}
    u_{t-1} &= \frac{\beta_t \sqrt{\gamma_{t-1}}}{1 - \gamma_t}u_0 + \frac{\sqrt{1 - \beta_t} (1 - \gamma_{t-1})}{1 - \gamma_t} u_t \\
    &+ \frac{\beta_t \sqrt{1 - \gamma_{t-1}}}{1 - \gamma_t}\xi - \frac{(1 - \gamma_{t-1}) \sqrt{\beta_t}\sqrt{1 - \beta_t}}{1 - \gamma_t} \xi'.
\end{align}

We now analyize the noise terms (i.e. only those terms depending on $\xi, \xi'$). By Lemmas \eqref{lemma:linear_gre}-\eqref{lemma:sum_of_gres} and the independence of $\xi, \xi'$, the sum of the noise terms follows a mean zero Gaussian measure with covariance 
\begin{align*}
    &\left(\frac{\beta_t \sqrt{1 - \gamma_{t-1}}}{1 - \gamma_t} \right)^2 + \left(\frac{(1 - \gamma_{t-1}) \sqrt{\beta_t} \sqrt{1 - \beta_t})}{1 - \gamma_t} \right)^2 \\
    &= \left(\frac{\beta_t(1 - \gamma_{t-1})}{1 - \gamma_t}\right)\left( \frac{\beta_t + (1-\beta_t)(1 - \gamma_{t-1})}{1 - \gamma_t}\right) \\
    &= \frac{\beta_t(1 - \gamma_{t-1})}{1 - \gamma_t}.
\end{align*}

where the last line follows from the calculation in Equation \eqref{eqn:step5}. Thus, we see that $u_{t-1} \mid u_t, u_0$ follows a Gaussian measure with the claimed mean and covariance.
\end{proof}

\subsection{The Feldman-H\'ajek Theorem and its Consequences}

Here we state the Feldman-H\'ajek Theorem in its general form, and discuss a few of its consequences. We include here only a statement of the theorem -- see \citet[Theorem~2.23, Theorem~2.25]{da2014stochastic} for a proof. 

\begin{theorem}[The Feldman-H\'ajek Theorem, General Case]
Let $\P = \NN(m_1, C_1)$ and $\Q = \NN(m_2, C_2)$ be Gaussian measures on $\FF$. Then,
\begin{enumerate}
    \item The measures $\P$ and $\Q$ are either equivalent (i.e. $\P \ll \Q$ and $\Q \ll \P$) or mutually singular (i.e. $\P$ is not absolutely continuous with respect to $\Q$ and vice-versa).
    \item The measures $\P$ and $\Q$ are equivalent if and only if:
    \begin{enumerate}
        \item $C_1^{1/2}(\FF) = C_2^{1/2}(\FF) = H_0$
        \item $m_1 - m_2 \in H_0$
        \item The operator $(C_1^{-1/2} C_2^{1/2})(C_1^{-1/2} C_2^{1/2})^* - I$ is Hilbert-Schmidt on the closure $\overline{H_0}$.
    \end{enumerate}
    \item If $\P$ and $\Q$ are equivalent and $C_1 = C_2 = C$, then $\Q$-a.s. the Radon-Nikodym derivative $\d \P / \d \Q$ is given by
    \begin{equation}
        \frac{\d \P}{\d \Q}(f) = \exp \left[\langle C^{-1/2}(m_1 - m_2), C^{-1/2}(f - m_2) \rangle - \frac{1}{2} ||C^{-1/2} m_1 - m_2)||^2 \right] \qquad \forall f \in \FF 
    \end{equation}
\end{enumerate}

\end{theorem}

An important consequence of the Feldman-H\'ajek theorem that we repeatedly make use of throughout our work is that it allows is to compute the KL between Gaussian measures having equal covariance operators.

\textbf{Proof of Proposition \eqref{prop:kl_gaussians}.}
\begin{proof}
    Suppose $\P = \NN(m_1, C)$ and $\Q = \NN(m_2, C)$ and that $m_1 - m_2 \in C^{1/2}(\FF)$. It follows from the Feldman-H\'ajek theorem that $\P$ and $\Q$ are equivalent. We now use the Radon-Nikodym expression from the Feldman-H\'ajek theorem to compute the KL divergence.
    
    We have that
    \begin{align}
        \KL{\P}{\Q} &= \int_\FF \log \frac{\d \P}{\d \Q} (f) \d \P(f) \\
        &= -\frac{1}{2} ||C^{-1/2} (m_1 - m_2))|| + \int_{\FF} \langle C^{-1/2} (m_1 - m_2) , C^{-1/2}(f - m_2) \d \P(f).
    \end{align}
    
    We now analyze the integral term via a spectral decomposition. Let $\{ (\lambda_j, e_j ) \}_{j=1}^\infty$ be the eigenvalues and eigenvectors of $C$. Note that the eigenvectors of $C$ form an orthonormal basis for $\FF$ by the spectral theorem, as $C$ is a self-adjoint compact operator. Then, we may evaluate the second integral as
    \begin{align}
        &\int_{\FF} \langle C^{-1/2} (m_1 - m_2) , C^{-1/2}(f - m_2) \d \P(f) \\
        &= \int_{\FF} \sum_{j=1}^\infty \langle m_1 - m_2, e_j \rangle \langle f - m_2, e_j \rangle \lambda_j^{-1} \d \P(f) \\
        &= \sum_{j=1}^\infty \lambda_{j}^{-1} \langle m_1 - m_2, e_j \rangle \int_{\FF} \langle f - m_2, e_j \rangle \d \P(f) \\
        &= \sum_{j=1}^\infty \lambda_{j}^{-1} \langle m_1 - m_2, e_j \rangle^2 \\
        &= \langle C^{-1/2} (m_1 - m_2), C^{-1/2}(m_1 - m_2) \rangle.
    \end{align}
    
    Combining this computation with the KL expression above completes the proof.
\end{proof}

\section{Loss Function}
\label{appendix_sect:loss_function}

In this section we provide additional details regarding the derivation and parametrization of our loss function.

\subsection{Functional ELBO}

\textbf{Proof of Proposition \eqref{prop:loss_function}.}
\begin{proof}
    First, we apply the usual functional ELBO \citep{wild2022generalized, matthews2016sparse, sun2019functional}, treating $u_{0:T}$ as latent variables and using the assumption that the reverse-time chain is Markov to obtain
    \begin{equation}
        \log q^\theta(\vec{y} \mid \vec{x}) \geq \E_{\P} \left[\log q^\theta(\vec{y} \mid \vec{x}, u_0) \right] - \KL{\P(\d u_{0:T} \mid \vec{x}, \vec{y})}{\Q^{\theta}(\d u_{0:T})}.
    \end{equation}
    
    By the chain rule for KL divergences \citep{dupuis2011weak}, we may condition on $u_T$ to obtain
    \begin{align}
       \log q^\theta(\vec{y} \mid \vec{x}) \geq \E_{\P} \left[\log q^\theta(\vec{y} \mid \vec{x}, u_0) \right] &- \KL{\P_T(\d u_T \mid \vec{x}, \vec{y}) }{\Q^\theta(\d u_T)} \\
       &- \E_{\P} \left[ \KL{\P(\d u_{0:T-1} \mid \vec{x}, \vec{y}, u_T) }{\Q^{\theta}(\d u_{0:T} \mid u_T)} \right]. \nonumber
    \end{align}

    Repeatedly applying the KL divergence chain rule to condition on $u_{T-1}, u_{T-2}, \dots, u_{1}$ and using the Mavkov assumption yields
    \begin{align}
        = \E_{\P} \left[\log q^\theta(y \mid x, u_0) \right] &- \KL{\P_T(\d u_T \mid \vec{x}, \vec{y} ) }{\Q^\theta(\d u_T)} \\
        & - \sum_{t=1}^T \E_{\P} \left[ \KL{\P(\d u_{t-1} \mid u_t, \vec{x}, \vec{y})}{\Q^\theta(\d u_{t-1} \mid u_t)} \right].
    \end{align}

\end{proof}

\subsection{Parametrization and Re-Weighting}

By Equation \eqref{eqn:loss_kl_term_no_simplification}, our loss function depends on terms of the form
\begin{align}
   L_{t-1}= \frac{1}{2 \widetilde{\beta}_t}|| C^{-1/2} (\widetilde{m}_t(u_t, u_0) - m_t^\theta(u_t)) ||_{\FF}^2.
\end{align}

That is, our model must predict the mean function $\widetilde{m}_{t}(u_t, u_0)$ given $(t, u_t)$. By Proposition \eqref{prop:conditional_reverse_measure} and Proposition \eqref{prop:forward_process_distributions},
\begin{equation}
    \widetilde{m}_t(u_t, u_0) = \frac{\sqrt{\gamma_{t-1}} \beta_t}{1 - \gamma_t} u_0 + \frac{\sqrt{1 - \beta_t} (1 - \gamma_{t-1})}{1 - \gamma_t} u_t
\end{equation}
\begin{equation}
    u_0 = \frac{1}{\sqrt{\gamma_t}} \left( u_t - \sqrt{1 - \gamma_t} \xi\right) 
\end{equation}
where $\xi \sim \NN(0, C)$. Combining these two expressions, we see that
\begin{align}
    \widetilde{m}_t(u_t, u_0) &= \frac{\sqrt{\gamma_{t-1}} \beta_t}{(1 - \gamma_t) \sqrt{\gamma_t}} \left(u_t - \sqrt{1 - \gamma_t} \right) + \frac{\sqrt{1 - \beta_t} (1 - \gamma_{t-1})}{1 - \gamma_t} u_t \\
    &= \frac{1}{\sqrt{1 - \beta_t}} \left(u_t - \frac{\beta_t}{\sqrt{1 - \gamma_t}} \xi \right).
\end{align}

We thus parametrize the variational mean via
\begin{equation}
    m_t^\theta(u_t) = \frac{1}{\sqrt{1 - \beta_t}} \left(u_t - \frac{\beta_t}{\sqrt{1 - \gamma_t}} \xi_t^\theta(u_t)  \right).
\end{equation}

Because $C$ is a linear operator, $C^{-1/2}$ must also be a linear operator. Thus, plugging in our reparametrized expressions for $\widetilde{m}_t(u_t, u_0)$ and $m_t^\theta(u_t)$, we see that
\begin{equation}
    L_{t-1} = \frac{\beta_t^2}{2 \widetilde{\beta}_t (1 - \beta_t) (1 - \gamma_t)} || C^{-1/2} (\xi - \xi_t^\theta(u_t)) ||_{\FF}^2. 
\end{equation}

We thus have obtained Equation \eqref{eqn:loss_kl_term}.

\section{Covariance Operators}
\label{appendix_sect:covariance_operators}

Recall that for a Gaussian measure $F \sim \P_F = \NN(m, C)$ on a separable Hilbert space $\FF$, the covariance operator $C: \FF \to \FF$ is defined via
\begin{equation} \label{eqn:appendix_covariance}
    Cg = \int_{\FF} \langle g, F \rangle F \d \P_F - \langle g, m \rangle m \qquad \forall g \in \FF.
\end{equation}

We now focus on the case where our Gaussian measure is specified by a Gaussian process with mean $m \in \FF$ and kernel $k: \XX^2 \to \R$. Note that $k(x, x') = \E[ \left(F(x) - m(x) \right) \left(F(x') - m(x') \right) ]$ specifies the covariance at points $x, x' \in \XX$. 

\subsection{Square-Integrable Case}

Consider first $\FF = L^2(\XX, \mu)$. In this setting, we have that
\begin{equation}
    [Cg](x) = \int_{\XX} k(x, x') g(x') \d \mu(x').
\end{equation}

This can be derived from Equation \eqref{eqn:appendix_covariance} via
\begin{align}
    [Cg](x) &= \int_{\FF} \langle g, F \rangle_{L^2(\XX, \mu)} F(x) \d \P_F - \langle g, m \rangle_{L^2(\XX, \mu)} m(x) \\
    &= \int_{\FF} \left[\int_{\XX} g(x') F(x') \d \mu(x') \right] F(x) \d \P_F - \langle g, m \rangle_{L^2(\XX, \mu)} m(x) \\
    &= \int_{\XX} g(x') \left[ \int_{\FF} F(x) F(x') \d \P_F \right]\d \mu(x')  - \langle g, m \rangle_{L^2(\XX, \mu)} m(x) \\
    &= \int_{\XX} g(x') \left[ k(x, x') + m(x) m(x') \right] \d \mu(x') - \int_{\XX} g(x') m(x') m(x) \d \mu(x') \\
    &= \int_{\XX} g(x') k(x, x') \d \mu(x').
\end{align}

where we apply Fubini's theorem in the third equality.

\vspace{1cm}

\textbf{Proof of Proposition \eqref{prop:kl_increasing}.}
\begin{proof}
Set $\vec{z} = \{ z^{(1)}, \dots, z^{(n)} \} \subset \XX$. If suffices to check the case that $\vec{x} = x \cup \vec{z}$ is increased by a single point $x \in \XX$. 

Let $K_{\vec{z} \vec{z}} \in \R^{n \times n}$ be the covariance matrix corresponding to $\vec{z}$, and let $K_{\vec{x}\vec{x}} \in \R^{(n+1)\times(n+1)}$ be the covariance matrix corresponding to $\vec{x}$, i.e. in both cases the covariance matrix is given by evaluating the kernel $k$ at all combinations of points in $\vec{z}$ or $\vec{x}$. Let $k_{\vec{z}}(x) \in \R^{n}$ be the covariance between the points of $\vec{z}$ and our new point $x$. Lastly, let $\Delta m (\vec{z}) \in \mathbb{R}^n$ be any vector and be $\Delta m (x) \in \R$ any scalar. We will write $\Delta m(\vec{x}) = [\Delta m(\vec{z}), \Delta m(x)]^T \in \mathbb{R}^{n+1}$ for the vector extending $\Delta m (\vec{z})$ by the single entry $\Delta m (x)$.

Then, we have that 
\begin{equation}
    K_{\vec{x} \vec{x}} = \begin{bmatrix} K_{\vec{z} \vec{z}} & k_{\vec{z}}(x) \\ k_{\vec{z}}(x)^T & k(x, x) \end{bmatrix},
\end{equation}
i.e. the extended covariance matrix corresponding to $\vec{x}$ can be written as a block matrix containing the covariance matrix for $\vec{z}$. Our goal is to show that
\begin{equation} \label{eqn:quadr_forms}
    \Delta m(\vec{z})^T K_{\vec{z} \vec{z}}^{-1} \Delta m (\vec{z}) \leq \Delta m (\vec{x})^T K_{\vec{x}\vec{x}}^{-1} \Delta(\vec{x} ).
\end{equation}

Using the block matrix inversion formula (see e.g. \citet[Appendix A.3]{williams2006gaussian}), we may express $K_{\vec{x}\vec{x}}^{-1}$ as
\begin{equation}
    K_{\vec{x}\vec{x}}^{-1} = \begin{bmatrix} K_{\vec{z}\vec{z}}^{-1} + K_{\vec{z}\vec{z}}^{-1} k_{\vec{z}}(x) M k_{\vec{z}}(x)^T K_{\vec{z}\vec{z}}^{-1} & -K_{\vec{z}\vec{z}}^{-1} k_{\vec{z}}(x) M \\
                                -M k_{\vec{z}}(x)^T K_{\vec{z}\vec{z}}^{-1} & M
    \end{bmatrix}
\end{equation}

where
\begin{equation}
    M = \left(k({x}, {x}) - k_{\vec{z}}(x)^T K_{\vec{z}\vec{z}}^{-1} k_{\vec{z}}(x) \right)^{-1} \in \R.
\end{equation}

Note that $M$ is exactly the posterior variance at $x \in \XX$ of a GP with covariance function $k$ \citep[Eqn. 2.26]{williams2006gaussian}. In particular, we must have $M \geq 0$. 

We now proceed to directly compute the quadratic form on the right-hand side of Equation \eqref{eqn:quadr_forms}. We have:
\begin{align}
    &\left[\Delta m(\vec{z}), \Delta m (x) \right] K_{\vec{x}\vec{x}}^{-1} \begin{bmatrix} \Delta m (\vec{z}) \\ \Delta m (x) \end{bmatrix} \\
    &= \left\langle \begin{bmatrix} \Delta m(\vec{z})^T \left(K_{\vec{z}\vec{z}}^{-1} + K_{\vec{z}\vec{z}}^{-1} k_{\vec{z}}(x) M k_{\vec{z}}(x)^T K_{\vec{z}\vec{z}}^{-1} \right) - \Delta m(x) M k_{\vec{z}}(x)^T K_{\vec{z}\vec{z}}^{-1} \\ -\Delta m(\vec{z})^T K_{\vec{z}\vec{z}}^{-1} k_{\vec{z}}(x)M + \Delta m(x)M  \end{bmatrix}, \begin{bmatrix} \Delta m(\vec{z}) \\ \Delta m(x) \end{bmatrix} \right\rangle \\
    &= \Delta m(\vec{z})^T K_{\vec{z}\vec{z}}^{-1} \Delta m(\vec{z}) + \Delta m(\vec{z})^T K_{\vec{z}\vec{z}}^{-1} k_{\vec{z}}(x) M k_{\vec{z}}(x)^T K_{\vec{z}\vec{z}}^{-1} \Delta m(\vec{z}) \\
    &\qquad \qquad - \Delta m(x) M k_{\vec{z}}(x)^T K_{\vec{z}\vec{z}}^{-1} \Delta m(\vec{z}) - \Delta m(\vec{z})^T K_{\vec{z}\vec{z}}^{-1} k_{\vec{z}}(x) M \Delta m(x) + \Delta m(x)^2 M \\
    &= \Delta m(\vec{z})^T K_{\vec{z}\vec{z}}^{-1} \Delta m(\vec{z}) + M\left( \Delta m(\vec{z})^T K_{\vec{z}\vec{z}}^{-1} k_{\vec{z}}(x) - \Delta m(x) \right)^2. \label{eqn:quadr_form_final}
\end{align}

We now plug Equation \eqref{eqn:quadr_form_final} back into Equation \eqref{eqn:quadr_forms}. Noting the first term in \eqref{eqn:quadr_form_final} is precisely the LHS of \eqref{eqn:quadr_forms}, we only need to check 
\begin{equation}
    0 \leq M\left( \Delta m(\vec{z})^T K_{\vec{z}\vec{z}}^{-1} k_{\vec{z}}(x) - \Delta m(x) \right)^2.
\end{equation}

However, note that we already observed that $M \geq 0$, and the other term is the square of a scalar, whence it is positive.
\end{proof}
\subsection{Sobolev Case}

We now consider the first-order Sobolev space $\FF = H^1(\XX, \mu)$.

We claim that
\begin{equation}
    [Cg](x) = \int_{\XX} k(x, x') g(x') \d \mu(x') + \int_{\XX} \partial_{x'} k(x, x') \partial_{x'} g(x') \d \mu(x')
\end{equation}
where we use the shorthand
\begin{equation}
    \partial_{x'} k(x, x') = \frac{\partial}{\partial x'} k(x, x') \quad \text{ and } \quad \partial_{x'} g(x') = \frac{\partial}{\partial x'} g(x').
\end{equation}

Note that the mean element is not dependent on the inner product -- it is merely an arbitrary element $m \in \FF$. Now, from Equation \eqref{eqn:appendix_covariance},
\begin{align}
    [Cg](x) &= \int_{\FF} \langle g, F \rangle_{H^1(\XX, \mu)} F(x) \d \P_F - \langle g, m \rangle_{H^1(\XX, \mu)} \\
    &= \int_{\FF} \left[\langle g, F \rangle_{L^2(\XX, \mu)} + \langle \partial_{x'} g(x'), \partial_{x'} F(x') \rangle_{L^2(\XX, \mu)} \right] F(x) \d \P_F - \langle g, m \rangle_{H^1(\XX, \mu)} m(x) \\
    &= \int_{\XX} k(x, x') g(x') \d \mu(x') + \int_{\FF} \langle \partial_{x'} g(x'), \partial_{x'} F(x') \rangle_{L^2(\XX, \mu)} F(x) \d \P_F - \langle \partial_{x'} g(x'), \partial_{x'} m(x') \rangle_{L^2(\XX, \mu)} m(x) \\
    &= \int_{\XX} k(x, x') g(x') \d \mu(x') + \int_{\XX} \partial_{x'} g(x') \E[F(x) \partial_{x'} F(x')] \d \mu(x') - \langle \partial_{x'} g(x'), \partial_{x'} m(x') \rangle_{L^2(\XX, \mu)} m(x) \\
    &= \int_{\XX} k(x, x') g(x') \d \mu(x') + \int_{\XX} \partial_{x'} g(x') \left( \partial_{x'} k(x, x') + m(x) \partial_{x'} m(x') \right)  \d \mu(x') \\
    &\qquad \qquad \qquad \qquad \qquad \quad - \langle \partial_{x'} g(x'), \partial_{x'} m(x') \rangle_{L^2(\XX, \mu)} m(x)  \nonumber \\
    &= \int_{\XX} k(x, x') g(x') \d \mu(x') + \int_{\XX} \partial_{x'} k(x, x') \partial_{x'} g(x, x') \d \mu(x').
\end{align}

The third equality follows from the corresponding $L^2(\XX, \mu)$ calculation. The fifth equality follows from the fact that if $F \sim \GP{m}{k}$ is differentiable with probability one, then $\partial_{x'} F$ is also a Gaussian process with mean $\partial_{x'} m$ \citep{williams2006gaussian, papoulis2002probability}, and moreover the covariance between $F$ and its derivative is given by differentiating the kernel:
\begin{equation}
    \text{Cov}(F(x), \partial_{x'} F(x') ) = \E\left[\left( F(x) - m(x) \right) \left(\partial_{x'} F(x') - \partial_{x'} m(x') \right)  \right] = \partial_{x'} k(x, x').
\end{equation}
See e.g. \citet[Chapter~9.4]{williams2006gaussian}.

\paragraph{PSD Projection Details}

In Sobolev space, our discrete approximation to the KL divergence is given in terms of a quadratic form (see Section \ref{sect:function_spaces})
\begin{align}
     \KL{\NN(m_1, C)}{\NN(m_2, C)} \approx \Delta m(\vec{x})^T \left[ I + D^T D \right] \left[ K_{\vec{x}\vec{x}} + K'_{\vec{x}\vec{x}} D \right]^{-1} \Delta m (\vec{x}). 
\end{align}

In practice, we parametrize $D$ by a first-order difference operator when $\XX = [0, 1]$. For example, one choice of $D$ can be constructed by using the forward difference equation at the left boundary, the backward difference equation at the right boundary, and the central difference equation on the interior of $[0, 1]$. 

Although the covariance operator $C$ is symmetric and positive semi-definite (\emph{with respect to the $H^1(\XX, \mu)$ inner product}) in theory, upon discretization we often obtain a non-PSD quadratic form. Thus, when naively used as a loss function, this quadratic form is unbounded from below, which leads to divergent training. We overcome this by projecting 
\begin{equation}
    A = \left[ I + D^T D \right] \left[ K_{\vec{x}\vec{x}} + K'_{\vec{x}\vec{x}} D \right]^{-1}
\end{equation}
to a symmetric PSD matrix. In particular, we apply the methods of \citet{cheng1998modified, higham1988computing} to find
\begin{equation}
    \widetilde{A} = \arg\min_{B} \{ ||B - A||_{F} : B \text{ is symmetric, PSD} \}
\end{equation}
i.e. the closest symmetric PSD matrix to $A$ in terms of the Frobenius norm. This has a unique solution, which can be computed in a straightforward manner. We briefly review this method here for the sake of completeness. First, set
\begin{equation}
    C = \frac{1}{2} (A + A^T)
\end{equation}
to be the symmetric part of $A$. Then, compute the usual spectral decomposition
\begin{equation}
    C = Q \text{diag}(\lambda_i) Q^T
\end{equation}
where $Q$ is a matrix containing the eigenvectors of $C$ with corresponding eigenvalues $\{ \lambda_i \}$. Set $\tau_i = \max(0, \lambda_i)$. Then,
\begin{equation}
    \widetilde{A} = Q \text{diag}(\tau_i) Q^T.
\end{equation}

We will write $\widetilde{A} = \pi_{\text{PSD}}(A)$ for this projection.

\section{Model Details}
\label{appendix_sect:model_details}

In all of our experiments, our model architecture is the Graph Neural Operator (GNO) of \citet{li2020neural}. We use a width of $64$, a kernel width of $256$, and a depth of $6$. Inputs to the GNO are graphs, constructed from discrete functional observations. In particular, for every function we construct a graph where each node corresponds to a single observation of the function. Each node has features corresponding to the observation location (i.e. point in $\XX$), function value (i.e. scalar in $\R$), and additionally time step $t \in [1, T]$. Nodes are connected if the Euclidean distance between their observation locations is smaller than a fixed radius $r$. We use $r=0.5$ in all of our experiments, and we additionally scale $\XX$ to $[0, 1] \subset \R$. Each edge in our graph has features corresponding to the observation locations and function values of the respective nodes. While using $r=1$ would be ideal in this setting, we find this to be prohibitively expensive in terms of computation and memory usage. The Fourier Neural Operator (FNO) \citep{li2021fourier} has a significantly reduced computation and memory cost compared to the GNO, but this model is limited to functional observations which are on a uniform gridding of $\XX$.

Our models are all trained for $50$ epochs and a learning rate of $0.001$.

We use $T = 1000$ time steps in all of our experiments. We set $\beta_1 = 10^{-4}$ and $\beta_T = 0.02$, and we linearly interpolate between these two values for other settings of $\beta_t$. We parametrize the Gaussian measure in our forward process via a mean-zero Gaussian process with a Mat\'ern kernel of unit variance and lengthscale $\ell = 0.1$. In particular, we use a Mat\'ern kernel with $\nu=1/2$ (i.e. the exponential kernel) when $\FF = L^2(\XX, \mu)$ and $\nu = 3/2$ when $\FF = H^1(\XX, \mu)$. This choice was made to ensure that the Gaussian measure was sufficiently rough to remove any information contained in the functional data, yet regular enough to be square-integrable (and differentiable in the $\nu = 3/2$ case) such that we obtain a valid Gaussian measure on $\FF$.

\subsection{Kernel Ablation}

In Tables \ref{tab:kernel_ablation_mogp}-\ref{tab:kernel_ablation_aemet}, we study the effect of the kernel choice in the forward process on the MoGP and AEMET datasets. In particular, we train models as above (using the discrete $L^2(\XX, \mu)$ loss function), but choose between values of $\nu = 1/2$ and $\nu = 3/2$ and sweep across various length scales between $0.005$ and $0.5$. We then sample 500 generated functions from our model, and compute the average pointwise mean and variance curves, as well as the average autocorrelation curve -- see Figures \eqref{fig:unconditional} and \eqref{fig:unconditional_appendix} for a visualization. We report the MSE between these generated functional statistics and the true functional statistics given by the training data.

We see that choosing a length scale that is either significantly larger or smaller than the length scale of the underlying functional data can have negative effects on the statistics, but for reasonable choices of the length scale, the statistics are comparable. Although $\ell = 0.1$ does not produce the best MSE values on the AEMET datsaet, we still use $\ell=0.1$ in our main experiments as this produced the most qualitatively realistic generated curves.

\begin{table}[!htb]
    \caption{Effect of kernel choice on the MoGP dataset. We report the MSE between various functional statistics on the training data and data generated via our model with the listed kernel hyperparameters.}
    \centering
    \begin{tabular}{lcccc}
    \toprule 
    $\nu$ & $\ell$ & Mean & Var. & Autocorr. \\
    \midrule 
    1/2 & 0.005 & 3.0333  & 6.1184 & 1.211e-4 \\
        & 0.01  & 0.4474  & 1.2174 & 5.552e-06 \\
        & 0.1   & 0.0032  & 0.2328 & 9.169e-06 \\
        & 0.2   & 0.4496  & 1.2752 & 9.183e-06 \\
        & 0.5   & 0.0318  & 0.2772 & 1.080e-05 \\
    \midrule
    3/2 & 0.005 & 0.5225  & 0.5783 & 3.638e-05 \\
        & 0.01  & 1.6557  & 4.7887 & 5.699e-05 \\
        & 0.1   & 0.4645  & 0.1239 & 1.928e-05 \\
        & 0.2   & 0.1046  & 0.2300 & 3.947e-06 \\
        & 0.5   & 0.2651  & 0.2586 & 6.677e-05 \\
    \bottomrule
    \end{tabular}
    \label{tab:kernel_ablation_mogp}
\end{table}

\begin{table}[!ht]
    \caption{Effect of kernel choice on the AEMET dataset. We report the MSE between various functional statistics on the training data and data generated via our model with the listed kernel hyperparameters. For $\nu=3/2$ with a length scale of $\ell = 0.5$ our training failed to produce a reasonable model.}
    \centering
    \begin{tabular}{lcccc}
    \toprule 
    $\nu$ & $\ell$ & Mean & Var & Autocorr \\
    \midrule 
    1/2 & 0.005 & 0.1118   & 74.8143 & 1.813-06 \\
        & 0.01  & 0.0646   & 2.2001  & 4.563e-06 \\
        & 0.1   & 0.7284   & 2.2519  & 5.805e-05 \\
        & 0.2   & 0.0152   & 1.0748  & 2.551e-06 \\
        & 0.5   & 0.0832   & 3.0590  & 1.516e-05 \\
    \midrule
    3/2 & 0.005 & 0.1393   & 8.5001        & 1.700e-05 \\
        & 0.01  & 0.0899   & 1.28130       & 5.638-06 \\
        & 0.1   & 0.9317   & 148.4542      & 0.0021 \\
        & 0.2   & 7.3748   & 15634.6889    & 0.0398 \\
        & 0.5   & \blank   & \blank        & \blank \\
    \bottomrule
    \end{tabular}
    \label{tab:kernel_ablation_aemet}
\end{table}

\section{Pseudocode}
\label{appendix_sect:pseudocode}

In this section we detail pseudocode for model training, unconditional sampling, and conditional sampling. Note that during training, we assume $u_0(\vec{x}) = \vec{y}$, i.e. we treat the observations as if they were noiseless. Thus the likelihood term $q(\vec{y} \mid \vec{x}, u_0)$ does not contribute to the loss, and we need only optimize the terms $L_{t-1}$ (see Equation \eqref{eqn:loss_kl_term}. Moreover, as mentioned in the main paper, we set $\lambda_t = 1$ as is standard in diffusion modeling \citep{ho2020denoising}.

Note that the given pseudocode for conditional generation covers both hard and soft conditioning. Hard conditioning is obtained when $n_{\text{free}} = 0$, and soft conditioning is obtained by setting $n_{\text{free}} \geq 1$, i.e. the parameter $n_{\text{free}}$ indicates how many generation steps are not conditioned on the given information.

\begin{algorithm}
\caption{Training Step}\label{alg:training_loop}
    Sample $(\vec{x}, \vec{y})$ from training data\;
    Sample $t$ uniformly from $\{2, \dots, T\}$\;
    Sample $\xi \sim \GP{0}{k}$, evaluated at $\vec{x}$ to obtain $\xi(\vec{x})$\;
    Construct $u_t \mid u_0$, evaluated at $\vec{x}$, via Lemma \eqref{prop:forward_process_distributions}: $u_t(\vec{x}) = \sqrt{\gamma_t} u_0(\vec{x}) + \sqrt{1 - \gamma_t} \xi(\vec{x})$ \;
    Compute model output $\xi^\theta(\vec{x} \mid u_t, t)$\; 
    Take a $\theta$-gradient step on $L_{t-1} = \left(\xi(\vec{x}) - \xi^\theta(\vec{x} \mid u_t, t) \right)^T A \left(\xi(\vec{x}) - \xi^\theta(\vec{x} \mid u_t, t) \right)$, where
    \begin{equation}
        A = \begin{cases}
            K_{\vec{x} \vec{x}}^{-1} & \FF = L^2(\XX, \mu) \\
            \pi_{\text{PSD}} \left( [I + D^T D][K_{\vec{x} \vec{x}} + K'_{\vec{x}\vec{x}} D]^{-1}\right) & \FF = H^1(\XX, \mu)
            \end{cases}
    \end{equation}
              
\end{algorithm}

\begin{algorithm}
    \caption{Unconditional Sampling}\label{alg:unconditional_sampling}
    Specify query points $\vec{x} \subset \XX$\;
    Sample $u_T \sim \GP{0}{k}$, evaluated at $\vec{x}$ to obtain $u_T(\vec{x})$\;
    \For{$t = T, T-1, \dots, 1$}{
        Sample $\xi_t \sim \GP{0}{k}$, evaluated at $\vec{x}$ to obtain $\xi_t(\vec{x})$\;
        $u_{t-1}(\vec{x}) \gets \frac{1}{\sqrt{1 - \beta_t}} \left(u_t(\vec{x}) - \frac{\beta_t}{\sqrt{1 - \gamma_t}} \xi^\theta(\vec{x} \mid u_t, t) \right) + \sqrt{\widetilde{\beta_t}} \xi_t(\vec{x})$ \;
    }
    Return $u_0(\vec{x})$
\end{algorithm}

\begin{algorithm}
    \caption{Conditional Sampling} \label{alg:conditional_sampling}
    Given: conditioning information $\DD = \{ (x_c^{(i)}, y_c^{(i)}) \}_{i=1}^{n_c} = \{ \vec{x}_c, \vec{y}_c \}$ \;
    Specify query points $\vec{x} = \{ x^{(1)}, x^{(2)}, \dots, x^{(n)} \} \subset \XX$ \;
    Create augmented support $\vec{z} = \{  x^{(1)}, x^{(2)}, \dots, x^{(n)}, x_c^{(1)}, \dots, x_c^{(n_c)} \}$\;
    Sample $u_T \sim \GP{0}{k}$, evaluated at $\vec{z}$ to obtain $u_T(\vec{z})$\;
    \For{$t = T, T-1, \dots, 1$}{
        Sample $\xi_t \sim \GP{0}{k}$, evaluated at $\vec{z}$ and $\vec{x}_c$ to obtain $\xi_t(\vec{z})$\;
        
        Sample reverse process unconditionally on $\vec{z}$:
        \begin{equation}
            \widetilde{u}_{t-1}(\vec{z}) \gets \frac{1}{\sqrt{1 - \beta_t}} \left(u_t(\vec{z}) - \frac{\beta_t}{\sqrt{1 - \gamma_t}} \xi^\theta(\vec{z} \mid u_t, t) \right) + \sqrt{\widetilde{\beta_t}} \xi_t(\vec{z}) \;
        \end{equation}
        
        \uIf{$t > n_{\text{free}}$}{
        Sample $\xi_t' \sim \GP{0}{k}$, and evaluate at $\vec{x}_c$ to obtain $\xi_t'(\vec{x}_c)$\;
        Perturb conditioning information via the forward process:
        \begin{equation}
            \vec{y}_{c,t} = \sqrt{\gamma_t} \vec{y}_c + \sqrt{1 - \gamma_t} \xi'_t(\vec{x_c})
        \end{equation}
        
        For each $x \in \vec{z}$, conditioned on perturbed conditioning information by setting 
        \begin{equation}
            u_{t-1}(x) = \begin{cases} \widetilde{u}_{t-1}(x) & x \notin \DD \\ y_{c,t}(x) & x \in \DD \end{cases}\;
        \end{equation}
        }
        \uElse{
        Do no conditioning: $u_{t-1}(\vec{z}) \gets \widetilde{u}_{t-1}(\vec{z})$\;
        }
    }
    Return $u_0$
\end{algorithm}


\section{Additional Experiments}
\label{appendix_sect:addl_experiments}

\subsection{Unconditional Samples}

In Figure \eqref{fig:unconditional_appendix}, we provide additional examples of our model on various datasets not discussed in the main paper. The first dataset (\emph{Linear}) is a synthetic dataset consisting of random linear functions $u_0(x) = ax + b$ where $a \sim \NN(2, 0.25^2)$ and $b \sim \NN(-1, 0.07^2)$. Note that, although the pointwise variance of the generated samples in this dataset appear to be significantly smaller than the that of the true samples, this is largely due to the small scale of the variance. The other datasets (\emph{Growth, Canadian, Octane}) are well-known functional data analysis datasets, which are available in the Python package scikit-fda \citep{skfda2019}.

\begin{figure}[!htb]
    \centering
    \includegraphics{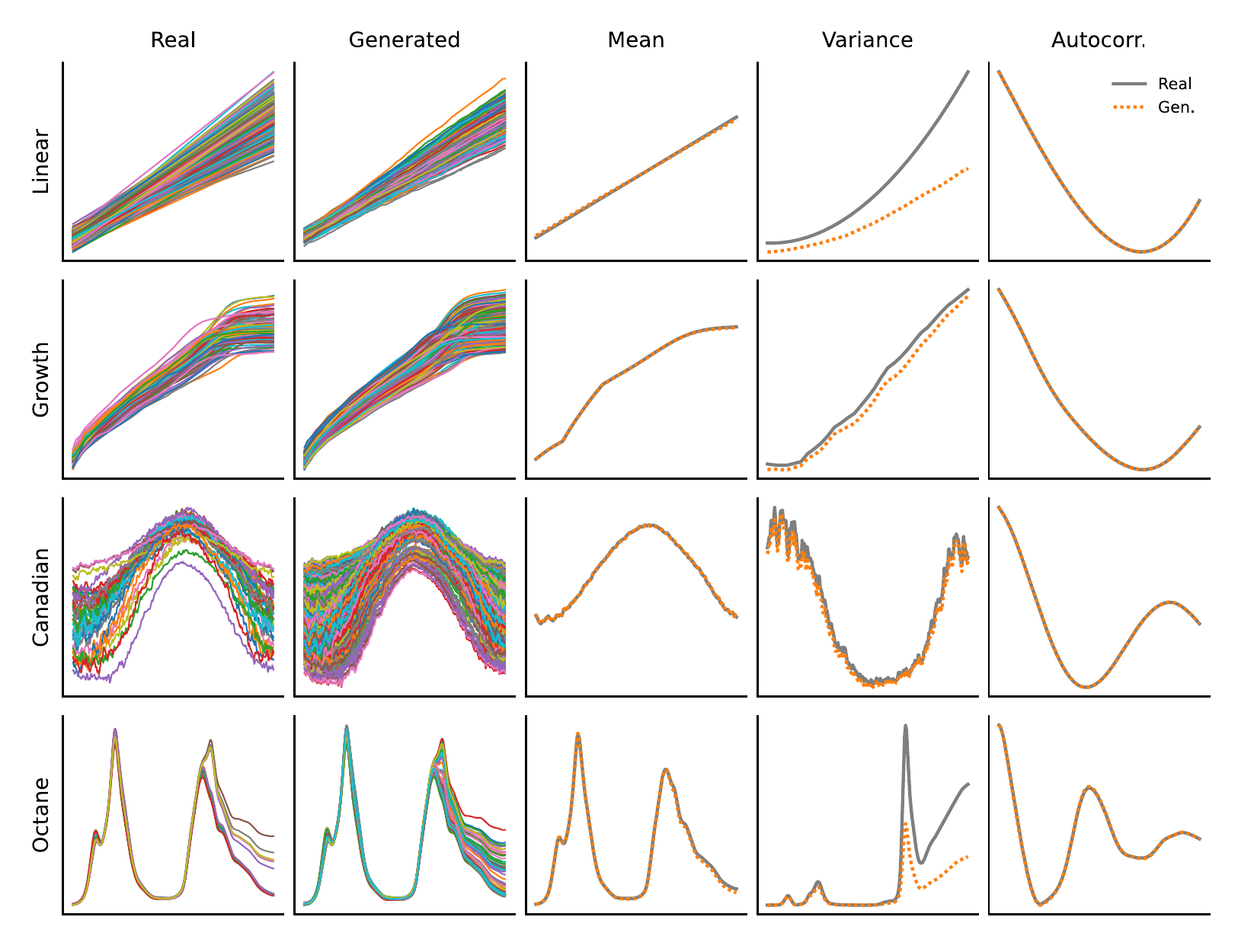}
    \caption{Unconditional function generation on a synthetic (Linear) and several real-world (Growth, Canadian, Octane) datasets. For each dataset, a GNO model was trained on the plotted functions (first column), and a total of $500$ functions were sampled from the model (second column).}
    \label{fig:unconditional_appendix}
\end{figure}


\subsection{FPCA Baseline}

We additionally include a simple unconditional baseline based on functional principal component analysis (FPCA). In particular, we approximate the first $M=5$ functional principal components by discretizing the training data (see \citet[Chapter~6]{ramsay2008functional} for details and \citet{skfda2019} for an implementation), followed by fitting a multivariate Gaussian to the resulting scores. To sample from this model, we sample from the Gaussian distribution over scores and project back to function space by taking linear combinations of the principal components with these sampled scores. 

See Figure \ref{fig:fpca_unconditional_appendix} for an illustration of this approach on all of the datasets we have thus far considered. We see that while the FPCA baseline is able to accurately match the functional statistics of the training data, the generated samples often fail to match the qualitative performance of our FuncDiff model (Figures \ref{fig:unconditional} and \ref{fig:unconditional_appendix}). Note that, unlike our FuncDiff model, we are unable to perform conditional generation with this FPCA baseline.

\begin{figure}[!hbt]
    \centering
    \includegraphics{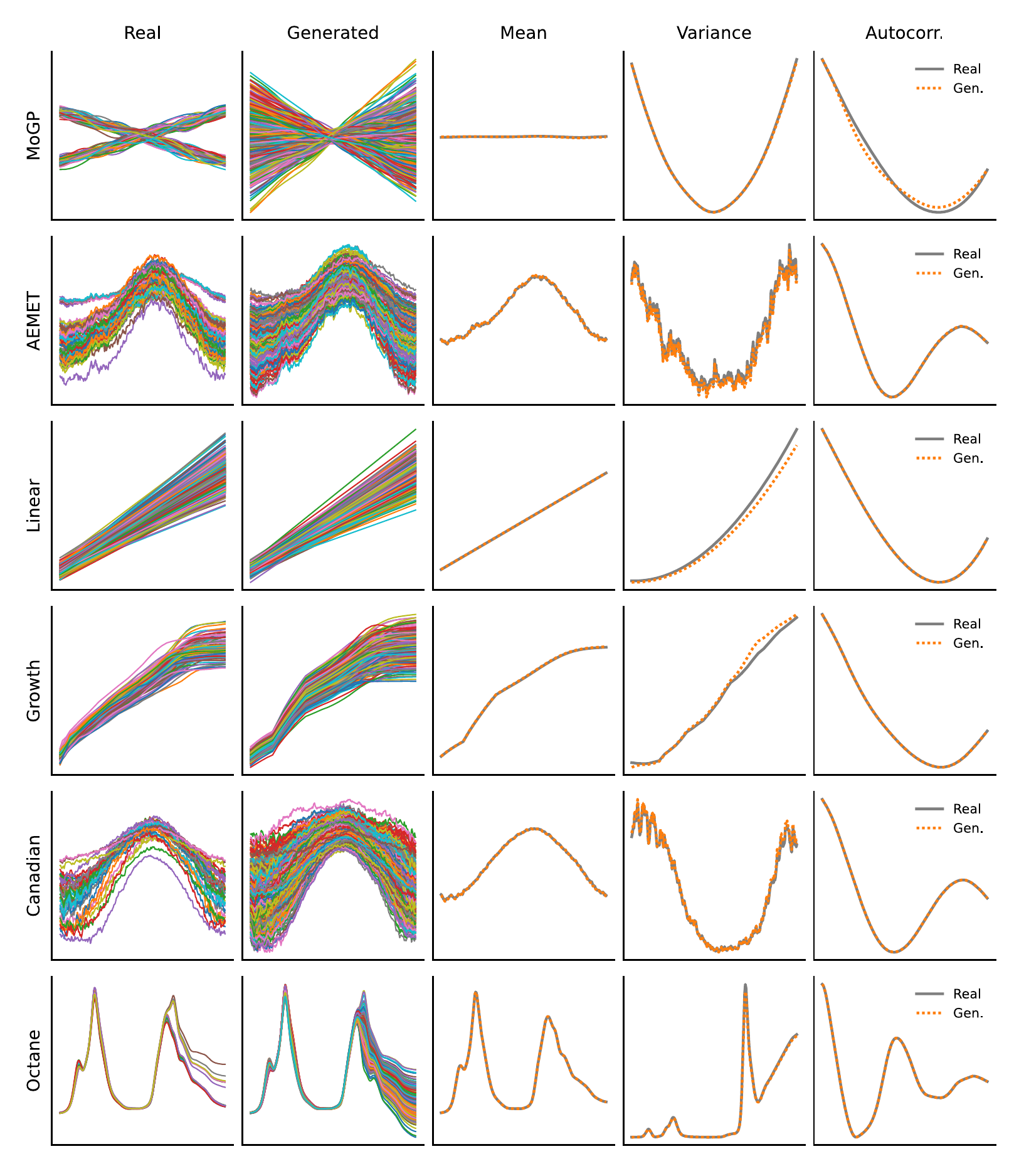}
    \caption{Unconditional samples from an FPCA-based model on various datasets. For each dataset, we estimate the first $M=5$ functional principal components and fit a Gaussian distribution to the resulting scores. Generation is performed by sampling from said Gaussian and taking the resulting linear combination of functional principal components. Although the functional statistics closely match those of the training data, the perceptual quality of the generated curves is worse than our FuncDiff model.}
    \label{fig:fpca_unconditional_appendix}
\end{figure}

\subsection{Spectral Loss}

In the main paper, we approximate the functional KL divergence by discretizing the underlying operators. In this section, we experiment with an alternative approach based on the spectrum of the covariance operator. We focus here on the setting $\FF = L^2(\XX, \mu)$ with $\XX = [0, 1]$ equipped with the Lebesgue measure $\mu = \d x$. Consider a Gaussian measure on $\FF$ with covariance operator $C$. Since $C$ is self-adjoint and compact, the spectral theorem tells us that the eigenfunctions of $C$ form an orthonormal basis of $\FF$. We denote the eigenvalues and eigenfunctions of $C$ by $\{ (\lambda_j, e_j) \}_{j=1}^\infty$. We then have that \citep[Remark~2.24]{da2014stochastic}
\begin{align}
    \KL{\NN(m_1, C)}{\NN(m_2, C)} &= \frac{1}{2} \langle m_1 - m_2, C^{-1} (m_1 - m_2) \rangle_{L^2(\XX, \mu)} \\
    &= \frac{1}{2} \sum_{j=1}^\infty \lambda_j^{-1} \langle m_1 - m_2, e_j \rangle_{L^2(\XX, \mu)}^2 \approx \frac{1}{2} \sum_{j=1}^J \lambda_j^{-1} \langle m_1 - m_2, e_j \rangle_{L^2(\XX, \mu)}^2.
\end{align}

Thus, an alternative method for approximating the KL divergence between Gaussian measures with equal covariance operators is to truncate the above sum at some specified number of terms $J$. For some choices of $C$, the eigenvalues and eigenfunctions are analytically known -- for example, see \citet[Chapter~4]{williams2006gaussian} for the squared-exponential kernel, and see \citet[Chapter~2]{le2010spectral} or \citet[Section~2.5]{burt2018spectral} for the exponential kernel.

In Figure \ref{fig:discrete_vs_spectral}, we compare this spectral approach to the discrete approach proposed in Section \eqref{sect:function_spaces}. In particular, we specify $C$ via a Gaussian process with a Mat\'ern kernel with $\nu = 1/2$, unit variance, and lengthscale $\ell = 0.1$. This is done to match the settings in our other experiments. Moreover, the eigenvalues and eigenfunctions are analytically available in this case \citep{le2010spectral, burt2018spectral}. In each row of Figure \ref{fig:discrete_vs_spectral}, we specify particular functions for $m_1$ and $m_2$. We vary the discretization size (i.e. the number of function observations) on the horizontal axis for discretization sizes of $10, 50, 100, 300$, and plot the estimated KL divergence between $\NN(m_1, C)$ and $\NN(m_2, C)$ on the vertical axis. 

We observe that the discrete approximation to the KL divergence (in blue) is monotonically increasing, as was proved in Proposition \eqref{prop:kl_increasing}. However, we see that the spectral approximation is sensitive to both the number of terms in the series expansion and the discretization size. In particular, when using $J=10$ terms, the spectral approximation underestimates the true KL divergence. In contrast, when $J \geq 50$, we see that the spectral approximation overestimates the true KL divergence by several orders of magnitude if the discretization of $\XX$ is not sufficiently fine. This effect worsens as we increase the number of terms $J$. We conjecture that this is because the eigenfunctions $e_j$ are sinusoidal in this case, and thus without a sufficiently fine discretization of $\XX$, the inner product in the spectral approximation is a poor numerical estimate of the true inner product. 

\begin{figure}[!htb]
    \centering
    \includegraphics{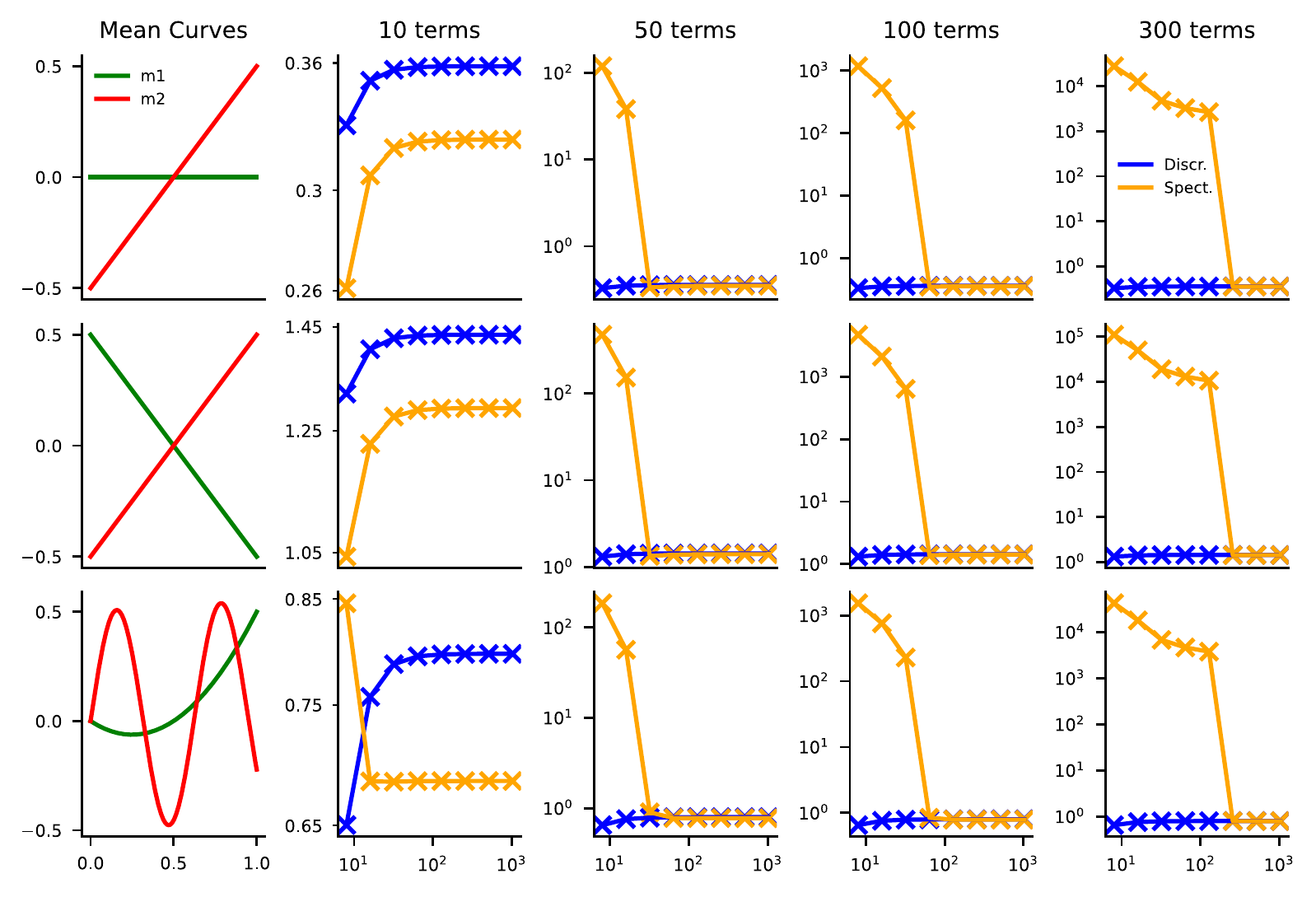}
    \caption{ Various synthetic functions (first column) and estimates of the KL divergence between Gaussian measures with these means, having covariance operator given by an exponential kernel. For columns 2-5, the horizontal axis corresponds to discretization size (i.e. number of function observations), and the vertical axis corresponds to the corresponding estimated KL divergence. The discrete method (in blue) has KL estimates that are monotonically increasing (see also Proposition \eqref{prop:kl_increasing}), but the spectral method (in orange) is sensitive to the choice of terms in the series expansion as well as the discretization size.}
    \label{fig:discrete_vs_spectral}
\end{figure}


\end{document}